\documentclass[twoside,11pt]{article}

\usepackage{jmlr2e}

\usepackage{blindtext}
\usepackage{hyperref}
\usepackage{graphicx}
\usepackage{epstopdf}
\usepackage{amsmath}
\usepackage{amssymb}
\usepackage{mathtools}
\usepackage[dvipsnames]{xcolor}
\usepackage{centernot}
\usepackage{nicefrac}
\usepackage{bbm}
\usepackage{thmtools}
\usepackage{thm-restate}
\usepackage{cleveref}
\usepackage{multirow}
\usepackage[parfill]{parskip}
\usepackage{longtable}
\usepackage{needspace}
\usepackage{centernot}
\usepackage{thmtools}
\usepackage{thm-restate}
\usepackage[algoruled,boxed,lined]{algorithm2e}
\usepackage{stmaryrd}
\usepackage{accents}
\usepackage{pdfpages}
\usepackage{fancyhdr}
\usepackage{eso-pic}
\usepackage{etoc, soul, lipsum}
\usepackage{lscape}
\usepackage{mathtools}

\newcommand*\diff{\mathop{}\!\mathrm{d}}
\newcommand{\z}{{\bf z}}
\newcommand{\x}{{\bf x}}

\newcommand{\ttt}{{\boldsymbol{\mathrm{t}}}}
\newcommand{\X}{\mathcal{X}}
\newcommand{\Y}{\mathcal{Y}}

\newcommand{\Hess}{{\bf H}}
\newcommand{\norm}[1]{\left\lVert#1\right\rVert}

\DeclareMathOperator*{\argmin}{arg\,min}

\newtheorem{thm}{Theorem}
\newtheorem{prop}{Proposition}
\newtheorem{assumption}{Assumption}

\usepackage{lastpage}
\jmlrheading{24}{2023}{1-\pageref{LastPage}}{8/23; Revised
11/23}{12/23}{23-1004}{Werner Zellinger, Stefan Kindermann, Sergei V.~Perverzyev}
\ShortHeadings{Adaptive Learning of Density Ratios in RKHS}{Zellinger, Kindermann and Pereverzyev}

\firstpageno{1}

\begin{document}

\title{Adaptive Learning of Density Ratios in RKHS\\~}

\author{\name Werner Zellinger \email werner.zellinger@oeaw.ac.at \\
       \addr Johann Radon Institute for Computational and Applied Mathematics\\
       Austrian Academy of Sciences\\
       Linz, Austria
       \AND
       \name Stefan Kindermann \email kindermann@indmath.uni-linz.ac.at \\
       \addr Industrial Mathematics Institute\\
       Johannes Kepler University Linz\\
       Linz, Austria
       \AND
       \name Sergei V.~Pereverzyev \email sergei.pereverzyev@oeaw.ac.at \\
       \addr Johann Radon Institute for Computational and Applied Mathematics\\
       Austrian Academy of Sciences\\
       Linz, Austria
       }

\editor{Mahdi Soltanolkotabi}

\maketitle

\begin{abstract}%
Estimating the ratio of two probability densities from finitely many observations of the densities is a central problem in machine learning and statistics with applications in two-sample testing, divergence estimation, generative modeling, covariate shift adaptation, conditional density estimation, and novelty detection.
In this work, we analyze a large class of density ratio estimation methods that minimize a regularized Bregman divergence between the true density ratio and a model in a reproducing kernel Hilbert space (RKHS).
We derive new finite-sample error bounds, and we propose a Lepskii type parameter choice principle that minimizes the bounds without knowledge of the regularity of the density ratio.
In the special case of square loss, our method adaptively achieves a minimax optimal error rate.
A numerical illustration is provided.
\end{abstract}


\section{Introduction}

\subsection{Problem}
\label{subsec:problem_setting}

Let $P$ and $Q$ be two probability measures on a measurable space $(\X,\mathcal{A})$ with compact $\X\subset\mathbb{R}^d$ and $\sigma$-algebra $\mathcal{A}$ of its measurable subsets.
Assume that $P$ is absolutely continuous with respect to $Q$.
Then, the Radon-Nikod\'ym theorem grants us a $Q$-integrable function $\beta:\X\to [0,\infty)$ satisfying, for any $U\in\mathcal{A}$,
\begin{align}
    \label{eq:RN_derivative}
    P(U)=\int_U \beta(x)\diff Q(x).
\end{align}
If both measures, $P$ and $Q$, have densities with respect to another reference measure, e.g., the Lebesgue measure, then $\beta$ equals the \textit{ratio of their densities}.

The problem of learning $\beta$ from samples
$\x:=(x_i)_{i=1}^m$ and $\x':=(x_i')_{i=1}^{n}$, drawn independently and identically (i.i.d.) from $P$ and $Q$, respectively,
plays a central role in statistics and machine learning~\citep{sugiyama2012density}; its applications include anomaly detection~\citep{smola2009relative,hido2011statistical}, two-sample testing~\citep{keziou2005test,kanamori2011f}, divergence estimation~\citep{nguyen2007estimating,nguyen2010estimating}, covariate shift adaptation~\citep{shimodaira2000improving,gizewski2022regularization,dinu2022aggregation}, generative modeling~\citep{mohamed2016learning}, conditional density estimation~\citep{schuster2020kernel}, and classification from
positive and unlabeled data~\citep{kato2019learning}.
Note, that this problem is different from supervised learning~\citep{Cucker2001OnTM} as the given data do not include point evaluations $\beta(x)$ from the desired target $\beta$.

It has been first observed by~\citet*{sugiyama2012densitybregman} that many algorithms for density ratio estimation compute finite-sample approximations of the objective $\beta$ by discretizing the minimization problem
\begin{align}
    \label{eq:regularized_Bregman_objective}
    \min_{f\in\mathcal{H}} B_F\!\left(\beta,g(f)\right) + \frac{\lambda}{2} \norm{f}^2,
\end{align}
where $g(f)$ is a model of the Radon-Nikod\'ym derivative $\beta$ with $f\in\mathcal{H}$
in a reproducing kernel Hilbert space (RKHS) $\mathcal{H}$,
$B_F(\beta,\widetilde{\beta}):=F(\beta)-F(\widetilde{\beta})-\nabla F(\widetilde{\beta})[\beta-\widetilde{\beta}]$ is a Bregman divergence generated by some convex functional $F:L^1(Q)\to\mathbb{R}$,
$\norm{.}$ is a norm on $\mathcal{H}$, and $\lambda\in\mathbb{R}$ is a regularization parameter; see Example~\ref{ex:example_in_introduction}.
For the exact form of the 
finite-sample approximations applied in practice, we ask the reader for patience until Eq.~\eqref{eq:empirical_loss_objective} is stated.

\begin{example}
\label{ex:example_in_introduction}
~
    \begin{itemize}
        \item As can be seen from~\citep[Section~3.2.1]{sugiyama2012densitybregman}, the kernel unconstrained least-squares importance fitting procedure (KuLSIF) in~\citep{kanamori2009least} uses Eq.~\eqref{eq:regularized_Bregman_objective} with $F(h)=\int_\X (h(x)-1)^2/2\diff Q(x)$ and $g=\mathrm{id}$, where $\mathrm{id}(f)=f$ is the identity (cf.~\citep[Eq.~(4)]{kanamori2012statistical}).
        This leads to a discretization of the minimization problem
        \begin{align}
            \label{eq:kulsif}
            \min_{f\in\mathcal{H}} \frac{1}{2} \norm{\beta-f}_{L^2(Q)}^2 + \frac{\lambda}{2} \norm{f}^2.
        \end{align}
        \item The logistic regression approach (LR) employed in~\citep[Section~7]{bickel2009discriminative} can be obtained from Eq.~\eqref{eq:regularized_Bregman_objective} using $F(h)=\int_\X h(x)\log(h(x))-(1+h(x))\log(1+h(x))\diff Q(x)$ and $g(f)=e^{f}$ with $f(x)=\theta^\text{T} \varphi(x)$ for some $\varphi:\X\to\mathbb{R}^d$; see~\citep[Section~3.2.3]{sugiyama2012densitybregman}.
        \item The exponential function approach (Exp) introduced in~\citep[Section~7]{menon2016linking} can be obtained using $F(h)=\int_\X h(x)^{-3/2}\diff Q(x)$ and $g(f)=e^{2 f}$; see~\citep[Table~1]{menon2016linking}.
        \item The square loss (SQ) is used in~\citep[Table~1]{menon2016linking} with $F(h)=\int_X 1/(2 h(x) + 2)\diff Q(x)$ and $g(f)=\frac{-1+2 f}{2-2 f}$.
    \end{itemize}
\end{example}
To the best of our knowledge, no error bounds are known for estimating Eq.~\eqref{eq:regularized_Bregman_objective} for the functionals in Example~\ref{ex:example_in_introduction} that take into account the regularity of the density ratio $\beta$, except~\citep[Section~3.1]{gizewski2022regularization} and~\citep{nguyen2023regularized} which hold only for KuLSIF.
Moreover, except for KuLSIF, no method for choosing the regularization parameter $\lambda$ in Eq.~\eqref{eq:regularized_Bregman_objective} is known that adapts to the regularity of $\beta$, i.e., a method with error bounds provably decreasing for increasing regularity.

\subsection{Results}

Following~\citet{menon2016linking}, one can relate the minimization of the data fidelity term $B(f):=B_F(\beta,g(f))$ from Eq.~\eqref{eq:regularized_Bregman_objective}, to the minimization of the expected loss with a loss function $\ell:\{-1,1\}\times\mathbb{R}\to\mathbb{R}$ measuring by $\ell(-1,f(x))$ and $\ell(1,f(x))$, respectively, the error of a classifier $f(x)$ which predicts whether $x$ is drawn from $P$ or $Q$.%

Our first result (Example~\ref{ex:self-concordant_losses}) shows that all loss functions $\ell$ corresponding to Example~\ref{ex:example_in_introduction}, satisfy a property called \textit{generalized self-concordance}~\citep{marteau2019beyond}.
This allows us to extend results from supervised learning in RKHSs to density ratio estimation.

Our second result (Proposition~\ref{prop:fast_rate}) provides, in a well specified setting, for large enough sample sizes and with high probability, an error bound
\begin{align}
\label{eq:result_rate}
    B(f_{\x,\x'}^{\lambda})\leq O\left( (m+n)^{-\frac{2 r \alpha+\alpha}{2 r \alpha +\alpha+1}}\right)
\end{align}
for a minimizer $f_{\x,\x'}^{\lambda}$ of the empirical estimation Eq.~\eqref{eq:empirical_loss_objective} of Eq.~\eqref{eq:regularized_Bregman_objective} with appropriately chosen regularization parameter $\lambda$.
The regularity index $r$ and the capacity index $\alpha$ in Eq.~\eqref{eq:result_rate} originate from source conditions as used in inverse problems and learning theory.
In the special case of square loss (SQ), the error rate in Eq.~\eqref{eq:result_rate} is minimax optimal in the same sense as it is minimax optimal in supervised learning, see~\citep{caponnetto2007optimal,blanchard2018optimal}.

Our third result (Eq.~\eqref{eq:balancing_principle_estimate}) is a new method for choosing the regularization parameter $\lambda$ that does not need access to the regularity index $r$ but nevertheless achieves the same rate as in Eq.~\eqref{eq:result_rate} (see Theorem~\ref{thm:balancing_principle_empirical}).
Our method is based on a Lepskii-type principle balancing bias and variance terms of a local quadratic approximation of the (generally not quadratic) data fidelity $B(f)$.
Our parameter choice strategy achieves the error rate in Eq.~\eqref{eq:result_rate} and is thus optimal in the case of square loss.
A numerical example illustrates our theoretical findings.

\subsection{Related Work}

The representation of density ratio estimation methods as minimizers of empirical estimations of Eq.~\eqref{eq:regularized_Bregman_objective} has been first observed by~\citet{sugiyama2012densitybregman} and later characterized in more detail by~\citet{menon2016linking}.
Our work extends~\citep{sugiyama2012densitybregman,menon2016linking} by taking into account the effect of the regularization term $\frac{\lambda}{2}\norm{f}^2$ in Eq.~\eqref{eq:regularized_Bregman_objective}.

In general, to the best of our knowledge, there are no error bounds similar to Eq.~\eqref{eq:result_rate} for the class of discretized minimization problems Eq.~\eqref{eq:regularized_Bregman_objective}.
For the special case of KuLSIF in Eq.~\eqref{eq:kulsif}, error rates have been discovered in~\citep{kanamori2012statistical,gizewski2022regularization,nguyen2023regularized}.
In particular, the error rate discovered in~\citep{kanamori2012statistical}, for RKHS with \textit{bracketing entropy} $\gamma\in(0,2)$, is
\begin{align}
\label{eq:kulsif_rate}
    B(f_{\z}^\lambda)=O\!\left(\min\{m,n\}^{-\frac{2}{2+\gamma}}\right).
\end{align}
This rate is worse than the one in Eq.~\eqref{eq:result_rate}, whenever $\alpha (2 r+1)>\frac{2}{\gamma}$.
A similar result for KuLSIF and the special case of Gaussian kernel, is presented in~\cite{que2013inverse}, see Type~I setting there.
We can also mention~\citep[Remark~3]{gizewski2022regularization} which provides the error bound
\begin{align}
    \label{eq:gizewski_fast_rate_intro}
    B(f_{\z}^\lambda)
    = O\!\left(\left(m^{-\frac{1}{2}}+n^{-\frac{1}{2}}\right)^{\frac{4r'+2}{2r'+2}}\right)
\end{align}
with $r'$ originating from a source condition similar to our Assumption~\ref{ass:source_condition} but for the covariance operator $T=\int_\mathcal{X} k(z,\cdot)\otimes k(z,\cdot)\diff Q(z)$ of $Q$ instead of $\rho$, where $k$ denotes the kernel of $\mathcal{H}$ and $(u\otimes v)[f]=\langle u,f\rangle v$ the rank-$1$ operator.
Our result in Eq.~\eqref{eq:result_rate} improves Eq.~\eqref{eq:gizewski_fast_rate_intro} for $\alpha\in(1,\infty)$ and $0<r'\leq r\leq\frac{1}{2}$.
In the recent work~\citep{nguyen2023regularized}, Eq.~\eqref{eq:gizewski_fast_rate_intro} is improved by additionally taking into account the capacity of the space $\mathcal{H}$ in terms of a new pointwise source condition on the reproducing kernel of $\mathcal{H}$ that allows to bound the regularized Christoffel function.
Such condition can be used as alternative for the effective dimension as defined in our Assumption~\ref{ass:capacity_condition}.
As a result, an error bound similar to our Eq.~\eqref{eq:result_rate} is obtained in~\citep{nguyen2023regularized}. However, the error is measured in the $\mathcal{H}$-norm and the bound holds only for the special case of the KuLSIF.


To the best of our knowledge, the balancing principle introduced in this work is the first parameter choice method for density ratio estimation with discovered error rate.
We can mention variants of cross-validation~\citep{kanamori2012statistical}, the quasi-optimality criterion~\citep{nguyen2023regularized}, and hints towards the use of balancing principles~\citep{gizewski2022regularization}, but, even in the special case of KuLSIF, we are not aware of any finite-sample error bounds related to these approaches.

Some of our proof techniques rely on~\citep{marteau2019beyond}, where  self-concordant loss functions in the context of kernel learning were studied first.
Our analysis extends~\citep{marteau2019beyond} to density ratio estimation and adaptive parameter choice.

\section{Notation and Auxiliaries}
\label{sec:preliminaries}

\subsection{Estimation of Bregman Divergence}
\label{subsec:class_probability_estimation}

\citet*{menon2016linking} study properties of $F$ and $g$, which allow an intuitive estimation of Eq.~\eqref{eq:regularized_Bregman_objective} by discriminating i.i.d.~observations $\x=(x_i)_{i=1}^m$ from $P$ and $\x'=(x_i')_{i=1}^n$ from $Q$.
In the following, we will summarize some of these properties.

Let us start by introducing the underlying model: a probability measure $\rho$ on $\X\times\Y$, $\Y:=\{-1,1\}$, with conditional probability measures\footnote{Existence follows from $\X\times\Y$ being Polish (cf.~\citep[Theorem~10.2.1]{dudley2018real}, the product case).} defined by $\rho(x|y=1):=P(x)$ and $\rho(x|y=-1):=Q(x)$ and a marginal (w.r.t.~$y$) probability measure $\rho_\Y$ defined as Bernoulli measure with probability $\frac{1}{2}$ for both events $y=1$ and $y=-1$.
We make the following assumption on the data generation process.%
\begin{assumption}[iid data]
    \label{ass:iid_of_classification_data}
    The data
    \begin{align}
    \z:=(x_i,1)_{i=1}^m\cup(x_i',-1)_{i=1}^n \in \X\times\Y   
    \end{align}
    is an i.i.d.~sample of an $\X\times\Y$-valued random variable $Z$ with measure $\rho$.
\end{assumption}
Assumption~\ref{ass:iid_of_classification_data} allows us to introduce a binary classification problem for the accessible data $\z$.
Recall the goal of the binary classification problem  is to find, based on the data $\z$, a classifier $f:\X\to\mathbb{R}$ with a small \textit{expected risk}
\begin{align}
    \label{eq:expected_risk}
    \mathcal{R}(f):=\mathbb{E}_{(x,y)\sim\rho}[\ell(y,f(x))]=\int_{\X\times\Y} \ell(y,f(x))\diff \rho(x,y)
\end{align}
for some loss function $\ell:\Y\times\mathbb{R}\to\mathbb{R}$.
In this work, we focus on \textit{strictly proper composite} loss functions with twice differentiable Bayes risk, which are characterized by the following assumption.

\begin{assumption}
    \label{ass:strictly_proper_composit_and_diff_bayes_risk}
    The function $\ell:\Y\times\mathbb{R}\to\mathbb{R}$ has an associated invertible link function $\Psi:[0,1]\to\mathbb{R}$ such that
    \begin{enumerate}
        \item the \textit{Bayes classifier} $f^\ast:=\argmin_{f:\X\to \mathbb{R}} \mathcal{R}(f)$ exists and $f^\ast(x)=\Psi\circ \rho(y=1|x)$,
        \item the \textit{Bayes risk} $G(u):= u \ell(1,\Psi(u))+(1-u) \ell(-1,\Psi(u))$ is twice differentiable.
    \end{enumerate}
\end{assumption}
The main result of~\citet{menon2016linking} is the following representation of the data fidelity in Eq.~\eqref{eq:regularized_Bregman_objective}, in terms of a loss function satisfying Assumption~\ref{ass:strictly_proper_composit_and_diff_bayes_risk}.
\begin{lemma}[{\citet{menon2016linking}, Proposition~3}]
\label{lemma:bregman_form_of_loss}
Let the Bregman generator $F$ and the density ratio model $g(f)$ be defined by
\begin{align}
    \label{eq:Bregman_generator_and_density_ratio_model}
    F(h):=-\int_\X (1+h(x))\cdot G\!\left(\frac{h(x)}{1+h(x)}\right)\diff Q(x),\quad g(f):=\frac{\Psi^{-1}\circ f}{1-\Psi^{-1}\circ f},
\end{align}
for some link function $\Psi$ and the Bayes risk $G$ of a loss $\ell$ satisfying Assumption~\ref{ass:strictly_proper_composit_and_diff_bayes_risk}. Then\begin{align}
    \label{eq:bregman_identity_for_density_ratio}
    B_F(\beta,g(f)) = 2 \left(\mathcal{R}(f)-\mathcal{R}(f^\ast)\right).
\end{align}
\end{lemma}
Indeed, for all approaches in Example~\ref{ex:example_in_introduction}, we can find $F$ and $g$ as required by Lemma~\ref{lemma:bregman_form_of_loss}, see~\citep[Table~1]{menon2016linking}.
We refer to~\citep{menon2016linking} for a recipe to construct new losses for Eq.~\eqref{eq:regularized_Bregman_objective} satisfying the requirements of Lemma~\ref{lemma:bregman_form_of_loss}.
\begin{example}
\label{ex:associated_loss_functions}%
~
\begin{itemize}
    \item KuLSIF~\citep{kanamori2009least} uses $\ell(-1,v)=\frac{1}{2}v^2, \ell(1,v)=-v$ and $\Psi(v)=\frac{v}{1-v}$.
    \item LR~\citep{bickel2009discriminative} uses $\ell(-1,v)=\log(1+e^v), \ell(1,v)=\log(1+e^{-v})$ and the link function $\Psi(v)=-\log(\frac{1}{v}-1)$.
    \item Exp~\citep{menon2016linking} uses $\ell(-1,v)=e^v, \ell(1,v)=e^{-v}$ and the link function $\Psi(v)=-\frac{1}{2}\log(\frac{1}{v}-1)$.
    \item SQ~\citep{menon2016linking} uses $\ell(-1,v)=(1+v)^2, \ell(1,v)=(1-v)^2$ and the link function $\Psi(v)=\frac{v+1}{2}$. 
\end{itemize}
\end{example}
According to Lemma~\ref{lemma:bregman_form_of_loss}, Eq.~\eqref{eq:regularized_Bregman_objective} can be expressed as regularized risk minimization objective, given $g$ and $F$ are constructed as specified in Lemma~\ref{lemma:bregman_form_of_loss}.
This motivates the estimation of Eq.~\eqref{eq:regularized_Bregman_objective} by an empirical risk minimization
\begin{align}
    \label{eq:empirical_loss_objective}
    \min_{f\in\mathcal{H}} \frac{1}{m+n}\sum_{(x,y)\in\z} \ell(y,f(x))+ \frac{\lambda}{2}\norm{f}^2.
\end{align}
Throughout this work, we will denote a minimizer of Eq.~\eqref{eq:empirical_loss_objective} by $f_{\z}^{\lambda}$.
The goal of this work is to analyze the error $B(f_{\z}^{\lambda})=B_F(\beta,g(f_{\z}^{\lambda}))$ of estimating $\beta$ by $g(f_{\z}^{\lambda})$.

\subsection{Learning in RKHS with Convex Losses}
\label{subsec:self_concordance}

Lemma~\ref{lemma:bregman_form_of_loss} identifies the methods in Example~\ref{ex:example_in_introduction} as risk minimization problems with loss functions as detailed in Example~\ref{ex:associated_loss_functions}.
We will see in Section~\ref{sec:error_bounds} that these loss functions satisfy the following property when applied to functions $f\in\mathcal{H}$ from an RKHS $\mathcal{H}$.%
\begin{assumption}[generalized self-concordance~\citep{marteau2019beyond}]%
    \label{ass:generalized_self_concordance}%
    \sloppy
    For any $z=(x,y)\in \X\times\Y$, the function $\ell_z:\mathcal{H}\to\mathbb{R}, \ell_z(f):=\ell(y,f(x)), f\in\mathcal{H}$, is convex, three times differentiable, and it exists a set $\varphi(z)\subset\mathcal{H}$ such that
    \begin{align}
    \label{eq:self_concordant}
        \forall f\in\mathcal{H},\forall h,k\in\mathcal{H}: |\nabla^3 \ell_z(f)[k,h,h]|
        \leq \nabla^2 \ell_z(f)[h,h]\cdot \sup_{p\in\varphi(z)} |k\cdot p|.%
    \end{align}%
\end{assumption}%
The definition of self-concordance originates from the analysis of Newton's method~\citep{nesterov1994interior} and has been generalized in~\citep{bach2010self,marteau2019beyond}.
Before we review error bounds, we follow~\citet{marteau2019beyond} and make the following technical assumptions.
\begin{assumption}[technical assumptions]
\label{ass:boundedness_defin0_existence}
For self-concordant $\ell_z(\cdot)$ with the set $\varphi(z)\subset\mathcal{H}$ in Eq.~\eqref{eq:self_concordant}, it holds:
\begin{enumerate}
    \item
    \sloppy
    There exists $R\geq 0$ such that for any realization $z$ of the random variable $Z$, $\sup_{g\in\varphi(z)}\norm{g}\leq R$ almost surely.
    \item The following quantities are almost surely bounded for any realization $z$ of the random variable $Z$: $|\ell_z(0)|, \norm{\nabla \ell_z(0)}, \mathrm{Tr}(\nabla^2 \ell_z(0))$.
    \item There exists a minimizer $f_{\mathcal{H}}:=\argmin_{f\in\mathcal{H}} \mathcal{R}(f)$.
\end{enumerate}
\end{assumption}
In the sequel, the symbols $\varphi(Z)$ and $\ell_Z$ mean the varieties of sets $\varphi(z)$ and functions $\ell_z$ corresponding to the realizations $z$ of the random variable $Z$.
The first point of Assumption~\ref{ass:boundedness_defin0_existence} is standard in learning theory~\citep{caponnetto2007optimal,steinwart2009optimal} and, together with point~2, assures that $\mathcal{R}(f)$ is well defined for all $f\in\mathcal{H}$.
Point 3 imposes that our model is \textit{well specified} in the sense that the function class $\mathcal{H}$ contains the minimizer of $\mathcal{R}$.
Under these assumptions, Theorem~3 in~\citep{marteau2019beyond} holds, which we state in our notation as follows:
\begin{lemma}
[\citet{marteau2019beyond}, Theorem~3]
\label{lemma:marteau_theorem3}
    Let Assumptions~\ref{ass:iid_of_classification_data},~\ref{ass:generalized_self_concordance},~\ref{ass:boundedness_defin0_existence} be satisfied, $\delta\in (0,\frac{1}{2}]$, and define $B_1$  and $B_2$ as 
    \begin{align*}
        B_1:=\sup_{\norm{f}\leq\norm{f_{\mathcal{H}}}}\sup_{z\in\mathrm{supp}(\rho)}\norm{\nabla\ell_z(f)},\quad
        B_2:=\sup_{\norm{f}\leq\norm{f_{\mathcal{H}}}}\sup_{z\in\mathrm{supp}(\rho)}\mathrm{Tr}(\nabla^2\ell_z(f)).
    \end{align*}
    Then, whenever $\lambda\leq B_2$ and $m+n$ is larger than
    \begin{align}
        \label{eq:sample_size_condition_lemma_marteau}
        \max\left\{512\norm{f_{\mathcal{H}}}^2R^2\log\left(\frac{2}{\delta}\right), 512 \log\left(\frac{2}{\delta}\right),24\frac{B_2}{\lambda}\log\left(\frac{8 B_2}{\lambda\delta}\right), 256\frac{R^2 B_1^2}{\lambda^2}\log\left(\frac{2}{\delta}\right)\right\},
    \end{align}
    a minimizer $f_{\z}^{\lambda}$ of Eq.~\eqref{eq:empirical_loss_objective} satisfies, with probability at least $1-2\delta$,
    \begin{align}
        \label{eq:marteau_bound_theorem2}
        \mathcal{R}(f_{\z}^{\lambda})-\mathcal{R}(f_{\mathcal{H}})\leq \frac{84 B_1^2}{\lambda (m+n)}\log\left(\frac{2}{\delta}\right)+2\lambda \norm{f_{\mathcal{H}}}^2.
    \end{align}
\end{lemma}
Lemma~\ref{lemma:marteau_theorem3} is a general result, which does not take into account cases where a regularity of $f_\mathcal{H}$ is assumed beyond $f_\mathcal{H}\in\mathcal{H}$ in Assumption~\ref{ass:boundedness_defin0_existence}.
Such regularity can be encoded by so-called \textit{source conditions} as used in the regularization theory for inverse problems~\citep{engl1996regularization,bauer2007regularization,caponnetto2007optimal,steinwart2009optimal,rudi2017generalization,blanchard2018optimal}.%
\begin{assumption}[source condition]
    \label{ass:source_condition}
    There exist $r\in (0,\frac{1}{2}], v\in\mathcal{H}$ such that $f_{\mathcal{H}}=\Hess(f_{\mathcal{H}})^r v$ with the expected Hessian $\Hess(f):=\mathbb{E}[\nabla^2 \ell_Z(f)]$.
\end{assumption}
In the case of square loss, Assumption~\ref{ass:source_condition} recovers the polynomial source condition $f_{\mathcal{H}}=T^r v$ with covariance operator $T=\int_\mathcal{X} k(z,\cdot)\otimes k(z,\cdot)\diff \rho(z)$, for the kernel $k$ of $\mathcal{H}$ and the rank-$1$ operator $(u\otimes v)[f]=\langle u,f\rangle v$.
Starting from $r=0$ (assuming only $f_{\mathcal{H}}\in\mathcal{H}$), the regularity of $f_\mathcal{H}$ increases with increasing $r$, essentially reducing the number of eigenvectors of $T$ needed to well approximate $f_\mathcal{H}$.

It is well known in learning theory~\citep{zhang2002effective,caponnetto2005empirical,caponnetto2007optimal} that the rate of convergence of regularized learning algorithms is not only influenced by the regularity of $f_\mathcal{H}$, but also by the \textit{capacity} of the space $\mathcal{H}$ measured by the so-called \textit{effective dimension}.
We therefore make the following assumption of~\citep[Assumption~7]{marteau2019beyond}, which generalizes~\citep[Definition~1]{caponnetto2007optimal}.%
\begin{assumption}[capacity condition]
    \label{ass:capacity_condition}
    There exist $\alpha\geq 1, Q_0\geq 0$ such that $\mathrm{df}_\lambda\leq Q_0 \lambda^{-\frac{1}{\alpha}}$ with the degrees of freedom
    \begin{align}
        \mathrm{df}_\lambda:=\mathbb{E}\left[\norm{\Hess_\lambda(f_\mathcal{H})^{-\frac{1}{2}}\nabla\ell_Z(f_\mathcal{H})}^2\right]
    \end{align}
    and $\Hess_\lambda(f):=\Hess(f)+\lambda I$.
\end{assumption}
The \textit{degrees of freedom} term $\mathrm{df}_\lambda$ appears as \textit{Fisher information} in the asymptotic analysis of M-estimation~\citep{van2000asymptotic,lehmann2006theory}, appears similarly in spline smoothing~\citep{wahba1990spline}, and reduces to the \textit{effective dimension}~\citep{caponnetto2005empirical,caponnetto2007optimal} $\mathrm{df}_\lambda=\mathrm{Tr}(T T_\lambda^{-1})$ with $T_\lambda:=T+\lambda I$ for the square loss.
In the latter case, Assumption~\ref{ass:capacity_condition} is equivalent to the condition $\sigma_j=O(j^{-\alpha})$ for a sequence $\sigma_1,\sigma_2,\ldots$ of non-increasing positive eigenvalues of $T$~\citep[Theorem~5]{guo2022capacity}.
That is, a bigger $\alpha$ implies that fewer eigenvectors are needed to approximate elements from the image space of $T$ with a given capacity.
On the other hand, from a regularization theory perspective,
Assumption~\ref{ass:capacity_condition} can be interpreted as specifying the \textit{ill-posedness} of the problem of inverting $T$ (cf.~\citep{blanchard2018optimal}).

Under Assumptions~\ref{ass:source_condition} and~\ref{ass:capacity_condition}, the following result is stated in~\citep{marteau2019beyond} and rephrased here in our notation.
See Subsection~\ref{subsec:proofs_balancing_principle_empirical} for more details on the derivation.%
\begin{lemma}[{\citet{marteau2019beyond}, Theorem~38}]
    \label{lemma:marteau_main_result}
    Let Assumptions~\ref{ass:iid_of_classification_data},~\ref{ass:generalized_self_concordance}--\ref{ass:capacity_condition} be satisfied, $\delta\in(0,\frac{1}{2}]$, and define $B_1$ and $B_2$ by
    \begin{align*}
        B_1^\ast:=\sup_{z\in\mathrm{supp}(\rho) } \mathrm{Tr}(\nabla \ell_z(f_\mathcal{H})),\quad B_2^\ast:=\sup_{z\in\mathrm{supp}(\rho) } \mathrm{Tr}(\nabla^2 \ell_z(f_\mathcal{H})),\quad  L:=\norm{f_\mathcal{H}}_{\Hess^{-2 r}(f_\mathcal{H})},
    \end{align*}
    where $\norm{f}_{\mathbf{A}}:=\norm{\mathbf{A}^{\frac{1}{2}}f}$.
    Whenever $0<\lambda\leq \min\{B_2^\ast,(2 L R \log\frac{2}{\delta})^{-1/r}, Q_0^{2\alpha} (B_2^\ast)^\alpha (B_1^\ast)^{-2 \alpha}\}$ and $m+n$ is larger than
    \begin{align}
        \label{eq:sample_size_condition_marteau_thm38}
        \max\left\{5184 \frac{B_2^\ast}{\lambda}\log\left(\frac{8 \cdot 414^2 B_2^\ast}{\lambda \delta}\right), \frac{1296 Q_0^2}{L^2 \lambda^{1+2 r+1/\alpha}} \right\},
    \end{align}
    then a minimizer $f_{\z}^{\lambda}$ of Eq.~\eqref{eq:empirical_loss_objective} satisfies, with probability at least $1-2\delta$,
    \begin{align}
        \mathcal{R}(f_{\z}^{\lambda})-\mathcal{R}(f_{\mathcal{H}})
        &\leq
        \norm{f_{\z}^{\lambda}-f_{\mathcal{H}}}_{\Hess(f_{\mathcal{H}})}^2
        \leq
        \norm{f_{\z}^{\lambda}-f_{\mathcal{H}}}_{\Hess_\lambda(f_{\mathcal{H}})}^2\nonumber\\
        &\leq
        414 \frac{Q_0^2}{(m+n) \lambda^{1/\alpha}}\log\left(\frac{2}{\delta}\right) + 414 L^2 \lambda^{1+2 r}
        \label{eq:marteau_bound_theorem38}.
    \end{align}
\end{lemma}
In the special case of square loss (SQ), it is possible to find $\lambda$ in
Lemma~\ref{lemma:marteau_main_result} that recovers the minimax optimal error rate discovered in~\citep{caponnetto2007optimal}; see~\citep[Corollary~39]{marteau2019beyond}.
However, \citep[Corollary~39]{marteau2019beyond} uses an a priori choice of $\lambda$ that depends on the regularity level $r$, which cannot be computed from the available samples.

\section{Error Rates under Regularity and Capacity}
\label{sec:error_bounds}

Let us start by Assumption~\ref{ass:iid_of_classification_data}.
The goal of this section is to analyze the error of a density ratio model $g(f_{\z}^{\lambda})$ with a minimizer $f_{\z}^{\lambda}$ of Eq.~\eqref{eq:empirical_loss_objective}, in terms of the data fidelity $B(f):=B_F(\beta, g(f))$ as specified in Eq.~\eqref{eq:Bregman_generator_and_density_ratio_model}.
We note that this setting is general, in the sense that all methods in Example~\ref{ex:example_in_introduction} are included.

Our analysis is based on the observation that all loss functions in Example~\ref{ex:associated_loss_functions} are self-concordant as defined in Assumption~\ref{ass:generalized_self_concordance}.%
\begin{example}
\label{ex:self-concordant_losses}
~
    \begin{itemize}
    \item KuLSIF~\citep{kanamori2009least} uses the losses $\ell_{(x,-1)}(f)=\frac{1}{2}f(x)^2, \ell_{(x,1)}(f)=-f(x)$ which are three times differentiable and Eq.~\eqref{eq:self_concordant} holds with $\varphi((x,y))=\{0\}\subset\mathcal{H}$.
    \item LR~\citep{bickel2009discriminative} is self-concordant with $\varphi((x,y))=\{y K_x\}$, see~\citep[Example~2]{marteau2019beyond} and~\citep{bach2010self}.
    \item Exp~\citep{menon2016linking} uses $\ell_{(x,y)}(f)=e^{y f(x)}$ which is self-concordant with $\varphi((x,y))=\{y K_x\}$.
    \item SQ~\citep{menon2016linking} uses losses which are three times differentiable and Eq.~\eqref{eq:self_concordant} holds with $\varphi((x,y))=\{0\}\subset\mathcal{H}$.
\end{itemize}
\end{example}
Before discussing faster error rates with more advanced regularity concepts, let us discuss the special role of the parameter $\lambda$ when chosen to balance the two terms in Eq.~\eqref{eq:marteau_bound_theorem2}.
As an intermediate result, we will provide first error rates for regularized losses beyond the KuLSIF approach in Eq.~\eqref{eq:kulsif}.

\subsection{Slow Error Rate with A Priori Parameter Choice}
\label{subsec:simple_error_bound}

We have observed above that many methods for density ratio estimation are based on self-concordant loss functions.
From Lemma~\ref{lemma:marteau_theorem3} we know that, for large enough sample size $m+n$, with probability larger than $1-2\delta$, a minimizer $f_{\z}^{\lambda}$ of Eq.~\eqref{eq:empirical_loss_objective}
satisfies the following bound:
\begin{align}
    \label{eq:bias_variance_simple}
    B(f_{\z}^{\lambda})-B(f_\mathcal{H})\leq S(m+n,\delta,\lambda)+A(\lambda),
\end{align}
where $$S(m+n,\delta,\lambda):=168 \frac{B_1^2}{\lambda (m+n)}\log\!\left(\frac{2}{\delta}\right) \quad \text{ and } \quad A(\lambda):=4\lambda \norm{f_\mathcal{H}}^2.$$
The question is: \textit{how to choose $\lambda$?}

One strategy in learning theory is to
use the value $\lambda^\ast$ which \textit{balances}~\citep{lepskii1991problem,goldenshluger2000adaptive,birge2001alternative,mathe2006lepskii,de2010adaptive,mucke2018adaptivity,blanchard2019lepskii,lu2020balancing,zellinger2021balancing}
\begin{align}
\label{eq:balance}
    \eta\cdot S(m+n,\delta,\lambda^\ast)=A(\lambda^\ast)
\end{align}
for some $\eta\geq 1$ ensuring that $\lambda^\ast$ is in the admissible set of values as, e.g., specified by Eq.~\eqref{eq:sample_size_condition_lemma_marteau}.
The value $\lambda^\ast$ in Eq.~\eqref{eq:balance} is achieved, due to the monotonicity of $S(m+n,\delta,\lambda)$ (decreasing) and $A(\lambda)$ (increasing, starting from $A(0)=0$).

The parameter choice $\lambda^\ast$ is particularly interesting, as it minimizes the right-hand side of Eq.~\eqref{eq:bias_variance_simple} up to a constant factor and, as a consequence, allows to achieve error rates of optimal order; see Subsection~\ref{subsec:quality_of_balancing_value}.%
\begin{prop}
    \label{prop:first_rate}
    Let Assumptions~\ref{ass:iid_of_classification_data}--\ref{ass:boundedness_defin0_existence} be satisfied, and let $\delta\in (0,\frac{1}{2}]$.
    Let further
    $\lambda^\ast$ be the solution of Eq.~\eqref{eq:balance} for $\eta:= \frac{256 R^2}{42}\norm{f_\mathcal{H}}^2$, which we assume to be larger than $1$.
    The there exists a quantity $C>0$ not depending on $m,n,\delta$,
    such that the minimizer $f_{\z}^{\lambda^\ast}$ of Eq.~\eqref{eq:empirical_loss_objective} satisfies
    \begin{align}
        \label{eq:balancing_value_rate}
        B(f_{\z}^{\lambda^\ast})-B(f_\mathcal{H})\leq 
        C (m+n)^{-\frac{1}{2}} \log^\frac{1}{2}\!\left(\frac{2}{\delta}\right)
    \end{align}
    with probability at least $1-2\delta$, for large enough sample size $m+n$, which depends on $\delta$.
\end{prop}
Explicit bounds on the sample size $m+n$ and the quantity $C$ are given in Eq.~\eqref{eq:sample_condition_a_priori_bound} and Eq.~\eqref{eq:exact_form_of_constant_a_priori_bound}, respectively.
The value for $\lambda^\ast$ specified in Proposition~\ref{prop:first_rate} and Eq.~\eqref{eq:exact_balance_value_simple_bound} can be computed \textit{a priori}.
That is, all information needed to achieve the error rate $(m+n)^{-\frac{1}{2}}$ is available.
However, this situation changes in Subsection~\ref{subsec:improved_error_bound} when the regularity of $f_\mathcal{H}$ is taken into account.

\begin{remark}
    Together with Example~\ref{ex:self-concordant_losses}, our Proposition~\ref{prop:first_rate} provides (to the best of our knowledge) the first error rate for regularized density ratio estimation methods beyond KuLSIF.
    At the same time, for {$B(f_{\z}^\lambda)=\norm{\beta-f_{\z}^\lambda}_{L^2(Q)}^2$} as used in KuLSIF, the error has been analyzed in~\citep{kanamori2012statistical} in terms of the so-called bracketing entropy $\gamma$ of $\mathcal{H}$.
    In our notation, this result tells us that, if $\gamma\in (0,2)$ and $\beta\in\mathcal{H}$ (so that $B(f_\mathcal{H})=0$), then, for $\epsilon$ satisfying $1-\frac{2}{2+\gamma} < \epsilon <1$ and $\lambda^{-1}=O\!\left(\min\{m,n\}^{1-\epsilon}\right)$, with probability $1-\delta$, it holds that
    \begin{align}
        \label{eq:guarantee_for_kulsif}
        B(f_{\z}^{\lambda})=O\!\left(\min\{m,n\}^{\epsilon-1}\right).
    \end{align}
    The rate in Eq.~\eqref{eq:guarantee_for_kulsif} is faster for smaller $\epsilon$ and tends, for RKHSs with $\gamma\to 2$, to its optimum $\min\{m,n\}^{-\frac{1}{2}}$, which is the same rate as given by our Proposition~\ref{prop:first_rate}.
    For $m\approx n$ as motivated by Assumption~\ref{ass:iid_of_classification_data}, our result therefore recovers the state of the art~\citep{kanamori2009least} for KuLSIF and RKHS with bracketing entropy $\gamma\to 2$.
    In the same case, the rate has been recently improved in~\citep{gizewski2022regularization} for regular $\beta$, which we discuss in the next Section~\ref{subsec:improved_error_bound}.
\end{remark}

\subsection{Fast Error Rate Requiring A Posteriori Parameter Choice}
\label{subsec:improved_error_bound}

The following result is based on Lemma~\ref{lemma:marteau_main_result} and it takes into account the unknown regularity of $f_\mathcal{H}$ by Assumption~\ref{ass:source_condition}.
\begin{prop}
    \label{prop:fast_rate}
    Let Assumptions~\ref{ass:iid_of_classification_data}--\ref{ass:capacity_condition} be satisfied for some $\alpha\leq \frac{1}{1-2 r}$, and let $\delta\in [\frac{2}{e^{1296}},\frac{1}{2}]$.
    Let further
    $\lambda^\ast$ be a solution of Eq.~\eqref{eq:balance} with $\eta:= 1296 \log^{-1}\!\left(\frac{2}{\delta}\right)$.
    Then there exists a quantity $C>0$ not depending on $m,n,\delta$, such that the minimizer $f_{\z}^{\lambda^\ast}$ of Eq.~\eqref{eq:empirical_loss_objective} satisfies
    \begin{align}
        \label{eq:balancing_value_rate_fast}
        B(f_{\z}^{\lambda^\ast})-B(f_\mathcal{H})\leq 
        C (m+n)^{-\frac{2 r\alpha+\alpha}{2 r \alpha+\alpha+1}}
    \end{align}
    with probability at least $1-2\delta$, for large enough sample size $m+n$, which depends on $\delta$.
\end{prop}
Explicit bounds on $m+n$ and $C$ are given in Eq.~\eqref{eq:requirement_sample_size_fast_rate_proposition} and Eq.~\eqref{eq:exact_form_of_constant_fast_bound}, respectively.
The value $\lambda^\ast$ in Proposition~\ref{prop:fast_rate}, as specified precisely in Eq.~\eqref{eq:lambda_star_in_fast_rate_proof}, depends on the regularity index $r$ and can therefore not be computed \textit{a priori}. A new method is required. 
\begin{remark}
    \label{remark:fast_rate}
    Together with Example~\ref{ex:self-concordant_losses}, Proposition~\ref{prop:fast_rate} provides (the to the best of our knowledge first) error rates result for regularized density ratio estimation methods that takes into account both, the regularity of the density ratio $\beta$ and the capacity of the space $\mathcal{H}$.
    As a result, even in the special case of KuLSIF with $B(f_{\z}^\lambda)=\norm{\beta-f_{\z}^\lambda}_{L^2(Q)}^2$, Proposition~\ref{prop:fast_rate} allows to refine existing results.
    To see this, consider the fastest error rate of~\citep{kanamori2012statistical} discovered for RKHS with bracketing entropy $\gamma\in(0,2)$ (cf.~Remark~1):
    \begin{align}
        B(f_{\z}^\lambda)=O\!\left(\min\{m,n\}^{-\frac{2}{2+\gamma}}\right).
    \end{align}
    This rate is improved by $O\!\left((m+n)^{-\frac{2 r\alpha+\alpha}{2 r \alpha+\alpha+1}}\right)$ in Eq.~\eqref{eq:balancing_value_rate_fast} whenever $\alpha (2 r+1)>\frac{2}{\gamma}$.
    We can also mention~\citep[Remark~3]{gizewski2022regularization} which leads to the error bound
    \begin{align}
        \label{eq:gizewski_fast_rate}
        B(f_{\z}^\lambda)-B(f_\mathcal{H})
        = O\!\left(\left(m^{-\frac{1}{2}}+n^{-\frac{1}{2}}\right)^{\frac{4r'+2}{2r'+2}}\right)
    \end{align}
    with $r'$ originating from a source condition similar to our Assumption~\ref{ass:source_condition} but for the covariance operator $T=\int_\mathcal{X} k(x,\cdot)\otimes k(x,\cdot)\diff Q(x)$ of $Q$ instead of $\rho$.
    The rate in Eq.~\eqref{eq:gizewski_fast_rate} is slower than the one in Eq.~\eqref{eq:balancing_value_rate_fast} for $\alpha\in(1,\infty)$ and $0<r'\leq r\leq\frac{1}{2}$.
    In general, for the square loss (SQ), the error rate in Eq.~\eqref{eq:balancing_value_rate_fast} is optimal in a minimax sense, see~\citep{caponnetto2007optimal,steinwart2009optimal,blanchard2018optimal}.
\end{remark}

\section{Balancing Principle using Approximation by Norm}
\label{sec:balancing_principle}

\subsection{Parameter Choice when the Norm is Known}
\label{subsec:balancing_principle_knonw_risk}

The value $\lambda^\ast$ in Proposition~\ref{prop:fast_rate} cannot be computed directly as it depends on the regularity of $\beta$ encoded in $r$; see Eq.~\eqref{eq:lambda_star_in_fast_rate_proof}.
One solution to this problem is to estimate $\lambda^\ast$ with Lepskii's \textit{balancing principle}~\citep{lepskii1991problem,goldenshluger2000adaptive,birge2001alternative,mathe2006lepskii,de2010adaptive,lu2020balancing,zellinger2021balancing}.
However, in general, no representation of $B(f_\z^\lambda)-B(f_\mathcal{H})$ as squared norm is available.
This prevents us from using the above referenced approaches directly, e.g., by using the norm estimates~\citep[Proposition~1]{de2010adaptive} or~\citep[Proposition~4.1]{lu2020balancing}. 

In the following, we propose a novel form of the balancing principle based on the norm $\norm{\cdot}_{\Hess_\lambda(f_\mathcal{H})}$ originating from the self-concordance-based upper bound in Eq.~\eqref{eq:marteau_bound_theorem38}.
Since the estimation of $\norm{\cdot}_{\Hess_\lambda(f_\mathcal{H})}$ requires some care, we postpone it to the next Subection~\ref{subsec:balancing_principle_empirical_risk}.

We start by defining a suitable discretization of increasing parameter candidates $(\lambda_i)_{i=1}^l\in\mathbb{R}$ such that $\lambda^\ast\in[\lambda_1,\lambda_l]$ is in the range of admissible values (ensured by multiplication of $\eta$ in Eq.~\eqref{eq:balance}).
The balancing principle is then based on checking a necessary condition.
Namely, whenever $\lambda_j\leq\lambda_i\leq\lambda^\ast$, Lemma~\ref{lemma:marteau_main_result} and the monotonicity of $$S(m+n,\delta,\lambda):=\frac{414 Q_0^2}{(m+n)\lambda^{1/\alpha}}\log\!\left(\frac{2}{\delta}\right) \quad \text{ and }  \quad A(\lambda):=414 L^2\lambda^{1+2 r},$$ grant us the inequality
\begin{align}
    \norm{f_\z^{\lambda_i}-f_\z^{\lambda_j}}_{\Hess_{\lambda_j}(f_\mathcal{H})}^2
    &\leq  2 \norm{f_\z^{\lambda_j}-f_{\mathcal{H}}}_{\Hess_{\lambda_j}(f_\mathcal{H})}^2 + 2 \norm{f_\z^{\lambda_i}-f_{\mathcal{H}}}_{\Hess_{\lambda_j}(f_\mathcal{H})}^2\nonumber\\
    &\leq  2 \norm{f_\z^{\lambda_j}-f_{\mathcal{H}}}_{\Hess_{\lambda_j}(f_\mathcal{H})}^2 + 2 \norm{f_\z^{\lambda_i}-f_{\mathcal{H}}}_{\Hess_{\lambda_i}(f_\mathcal{H})}^2\nonumber\\
    &\leq 2 S(m+n,\delta,\lambda_j) + 2 A(\lambda_j) + 2 S(m+n,\delta,\lambda_i) + 2 A(\lambda_i)\nonumber\\
    &\leq 4 S(m+n,\delta,\lambda_j) + 4 S(m+n,\delta,\lambda_i)\nonumber\\
    &\leq 8 S(m+n,\delta,\lambda_j)\nonumber\\
    \label{eq:necessary_balancing_condition}
    &\leq 8 \eta\cdot S(m+n,\delta,\lambda_j),
\end{align}
where we have used that $\Hess_{\lambda_j}(f)=\Hess(f)+\lambda_j I\leq \Hess(f)+\lambda_i I=\Hess_{\lambda_i}(f)$ and $\eta$ as defined in Proposition~\ref{prop:fast_rate}.
Following the above reasoning, if the error bound in Lemma~\ref{lemma:marteau_main_result} is sharp, then the maximum $\lambda_i\in (\lambda_i)_{i=1}^l$ satisfying the above condition for all smaller $\lambda_j, j\in \{1,\ldots,i-1\}$, can be expected to be close to $\lambda^\ast$.
This leads to the following estimate
\begin{align}
    \label{eq:balancing_principle_estimate_known_norm}
    \lambda_{+}:=\max\!\left\{\lambda_i: \norm{f_\z^{\lambda_i}-f_\z^{\lambda_j}}_{\Hess_{\lambda_j}(f_\mathcal{H})}^2\leq 8\eta\cdot S(m+n,\delta,\lambda_j), j\in\{1,\ldots,i-1\}\right\}.
\end{align}
Importantly, Eq.~\eqref{eq:necessary_balancing_condition} does not depend on the unknown regularity index $r\in (0,\frac{1}{2}]$ but only on the quantities $m,n,Q_0,\alpha$ which are either given or an upper bound can be easily estimated from data.
We may state the following preparatory result.
\begin{thm}
    \label{thm:balancing_principle_known_norm}
    Let Assumptions~\ref{ass:iid_of_classification_data}--\ref{ass:capacity_condition} be satisfied and let
    \begin{align*}
        \lambda_i:=\lambda_0\cdot \xi^{i}, i\in\{1,\ldots,l\}
    \end{align*}%
    with $\xi>1$, $\lambda_0\leq \xi^{-l} \min\left\{B_2^\ast,\left(2 L R \log\frac{2}{\delta}\right)^{-\frac{1}{r}}, Q_0^{2\alpha} (B_2^\ast)^\alpha (B_1^\ast)^{-2 \alpha}\right\}$, $l<\frac{e^{1296}}{4}-2$ and $\delta\in [\frac{2}{e^{1296}},\frac{1}{4+2l}]$.
    Then there exists a quantity $C>0$ not depending on $m,n,\delta$, such that the minimizer $f_{\z}^{\lambda_+}$ of Eq.~\eqref{eq:empirical_loss_objective} satisfies
    \begin{align}
        \label{eq:theorem_known_norm_rate_fast}
        B(f_{\z}^{\lambda_+})-B(f_\mathcal{H})\leq 
        C (m+n)^{-\frac{2 r\alpha+\alpha}{2 r \alpha+\alpha+1}}
    \end{align}
    with probability at least $1-(4+2 l)\delta$, for large enough sample size $m+n$ depending on $\delta$.
\end{thm}
Bounds on the sample size $m+n$ and the quantity $C$ are given in Eq.~\eqref{eq:sample_size_known_norm} and Eq.~\eqref{eq:const_in_known_risk_balancing_bound}, respectively.
We note that Theorem~\ref{thm:balancing_principle_known_norm} achieves the same error rate as Proposition~\ref{prop:fast_rate} but uses the parameter choice $\lambda_+$ which is a posteriori computable if we have access to the norm values $\norm{f_\z^{\lambda_i}-f_\z^{\lambda_j}}_{\Hess_{\lambda_j}(f_\mathcal{H})}^2$.
However, these values are not accessible, as we don't have access to the measure $\rho$.
This issue is solved in the next Subsection~\ref{subsec:balancing_principle_empirical_risk}.

\subsection{Parameter Choice with Estimated Norm}
\label{subsec:balancing_principle_empirical_risk}

The following concentration result is proven in Subsection~\ref{subsec:proofs_balancing_principle_empirical} and essentially combines analytical arguments for self-concordance losses~\citep{marteau2019beyond} with concentration inequalities for Hermitian operators~\citep{rudi2017generalization}.%
\begin{lemma}
    \label{lemma:Hessian_concentration}
    Let the assumptions of Lemma~\ref{lemma:marteau_main_result} be satisfied, let $\delta\in(0,\frac{1}{2}]$, denote the empirical Hessian by $\widehat{\Hess}(f):=\frac{1}{m+n}\sum_{z\in \z} \nabla^2 \ell_z(f)$, and let us write $\mathbf{B}\preccurlyeq\mathbf{A}$ iff $\mathbf{A}-\mathbf{B}$ is positive semi-definite.
    Whenever $\norm{\Hess(f_\z^\lambda)}>0$ and
    \begin{align}
        \label{eq:additional_sample_size_condition}
        m+n\geq \frac{16 B_2^\ast}{\norm{\Hess(f_\z^\lambda)}}\log\!\left(\frac{2}{\delta}\right),
    \end{align}
    then, with probability at least $1-4\delta$,
    \begin{align}
        \label{eq:Hessian_concentration}
        \widehat{\Hess}_\lambda(f_\z^\lambda)
        \preccurlyeq 6 \Hess_\lambda(f_\mathcal{H})
        \preccurlyeq 48 \widehat{\Hess}_\lambda(f_\z^\lambda).
    \end{align}
\end{lemma}
Let us denote by $(\lambda_i)_{i=1}^l\in\mathbb{R}$ a geometric sequence with $\lambda_i:=\lambda_0\cdot \xi^{i}, i\in\{1,\ldots,l\}$ as defined in Theorem~\ref{thm:balancing_principle_known_norm}.
Then, Lemma~\ref{lemma:Hessian_concentration} motivates the following assumption
\begin{assumption}
    \label{ass:Hessian_of_sequence}
    For the sequence $(\lambda_i)_{i=1}^l$
    there exists $b^\ast>0$ with $b^\ast\leq \min_{i\in\{1,\ldots,l\}}\norm{\Hess(f_\z^{\lambda_i})}$.
\end{assumption}
Note that for strictly convex 
loss functions, the Hessian is positive definite, and thus Assumption~\ref{ass:Hessian_of_sequence} is necessarily satisfied.

The same reasoning as done for Eq.~\eqref{eq:necessary_balancing_condition} motivates, under Assumption~\ref{ass:Hessian_of_sequence},
the estimate
\begin{align}
    \label{eq:balancing_principle_estimate}
    \lambda_{\mathrm{BP}}:=\max\!\left\{\lambda_i: \norm{f_\z^{\lambda_i}-f_\z^{\lambda_j}}_{\widehat{\Hess}_{\lambda_j}(f_\z^{\lambda_j})}^2\leq 48\eta\cdot S(m+n,\delta,\lambda_j), j\in\{1,\ldots,i-1\}\right\}
\end{align}
with $\eta:=1296\log^{-1}(\frac{2}{\delta})$ as defined in Proposition~\ref{prop:fast_rate}.
Note that Eq.~\eqref{eq:balancing_principle_estimate} uses only given data or information for which upper bounds are easily estimable.
Moreover, the following error rates result holds.
\begin{thm}
    \label{thm:balancing_principle_empirical}
    Let all assumptions of Theorem~\ref{thm:balancing_principle_known_norm} be satisfied for some $l<\frac{e^{1296}}{12}-1$ and $\delta\in [\frac{2}{e^{1296}},\frac{1}{6+6l}]$.
    Let further $(\lambda_i)_{i=1}^l$ satisfy Assumption~\ref{ass:Hessian_of_sequence}.
    Then there exists a quantity $C>0$ not depending on $m,n,\delta$, such that the minimizer $f_{\z}^{\lambda_{\mathrm{BP}}}$ of Eq.~\eqref{eq:empirical_loss_objective} satisfies
    \begin{align}
        \label{eq:theorem_empirical_norm_rate_fast}
        B(f_{\z}^{\lambda_{\mathrm{BP}}})-B(f_\mathcal{H})\leq 
        C (m+n)^{-\frac{2 r\alpha+\alpha}{2 r \alpha+\alpha+1}}
    \end{align}
    with probability at least $1-(6+6 l)\delta$, for large enough sample size $m+n$ depending on $\delta$.
\end{thm}
Exact bounds on the sample size $m+n$ and the quantity $C$ are given in Eq.~\eqref{eq:sample_size_empirical_norm} and Eq.~\eqref{eq:const_in_empirical_risk_balancing_bound}, respectively.
Theorem~\ref{thm:balancing_principle_empirical} is the main result of this work.
For large enough sample sizes, it provides an a posteriori parameter choice $\lambda_\mathrm{BP}$ that allows the minimizer $f_\z^{\lambda_\mathrm{BP}}$ of Eq.~\eqref{eq:empirical_loss_objective} to achieve an error rate that improves with the regularity of the true solution.
\begin{remark}
    To the best of our knowledge, Theorem~\ref{thm:balancing_principle_empirical} gives the first proof of the adaptivity to the unknown target regularity, for a parameter choice strategy in density ratio estimation.
    Even if similar adaptivity results hold for cross-validation in supervised learning~\citep{caponnetto2010cross} and the aggregation method in covariate shift domain adaptation~\citep{gizewski2022regularization}, their extension to density ratio estimation is not straight forward and an open future problem.
    In the special case of square loss (SQ), the order of the error rate is optimal (see Remark~\ref{remark:fast_rate}),  and Eq.~\eqref{eq:balancing_principle_estimate} thus provides, in this sense, adaptivity at optimal rate.
\end{remark}

\section{Numerical Example}
\label{sec:numerical_examples}

\begin{figure}[ht]
\includegraphics[width=.3\linewidth]{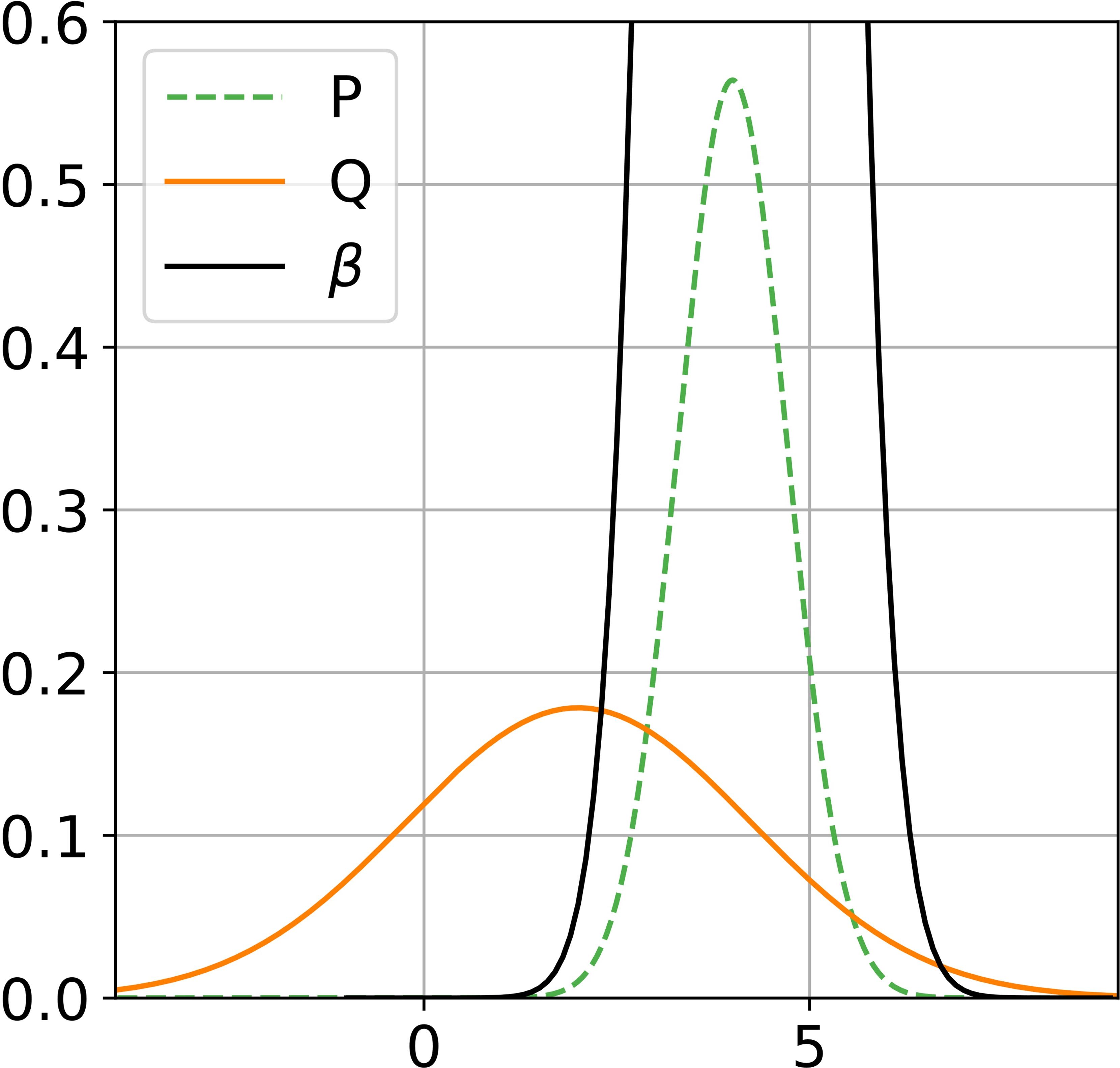}
\centering
\caption{The goal is to learn the density ratio $\beta=\frac{\diff P}{\diff Q}$ (black, solid).}
\label{fig:gaussians}
\end{figure}

To illustrate the behaviour of our approach, we rely on an example of~\citep{shimodaira2000improving,sugiyama2007covariate,dinu2022aggregation,nguyen2023regularized} with data from two Gaussian distributions $P,Q$ having means $4,2$ and standard deviations $1/\sqrt{2},\sqrt{5}$, respectively, see Figure~\ref{fig:gaussians}.

We use two different density ratio estimation methods in Eq.~\eqref{eq:empirical_loss_objective}, the Exp approach, and, the KuLSIF approach as described in Example~\ref{ex:example_in_introduction}.
In particular, we approximate a minimizer of Eq.~\eqref{eq:empirical_loss_objective} by the conjugate gradient method applied to the parameters $\alpha_1,\ldots,\alpha_{m+n}$ of Eq.~\eqref{eq:empirical_loss_objective} when evaluated at a model ansatz $f_\z^{\lambda}=\sum_{j=1}^{m+n} \alpha_j k(x_j,\cdot)$ with reproducing kernel $k(x,x'):=1+e^{-\frac{(x-x')^2}{2}}$ of $\mathcal{H}$, as suggested by the representer theorem~\citep{scholkopf2001generalized}.

Following our goal of choosing $\lambda$, we fix a geometric sequence $(\lambda_i)_{i=1}^l\in\mathbb{R}$ with $\lambda_i:=10^{i}, i\in\{-3,\ldots,1\}$ and compute
\begin{align}
    \label{eq:numerical_example_bp}
    \lambda^\circ:=\max\!\left\{\lambda_i: \norm{f_\z^{\lambda_i}-f_\z^{\lambda_j}}^2_{\widehat{\Hess}_{\lambda_j}(f_\z^{\lambda_j})}\leq \frac{M}{ \lambda_j (m+n)}, j\in\{1,\ldots,i-1\}\right\}
\end{align}
for some $M\in\mathbb{R}$ and the pessimistic choice $\alpha=1$ in Assumption~\ref{ass:capacity_condition}.
The empirical norm in Eq.~\eqref{eq:numerical_example_bp} can be computed, for two minimizers $f_\z^{\lambda_s}=\sum_{i=1}^{m+n} \alpha_i k(x_i,\cdot)$ and $f_\z^{\lambda_t}=\sum_{i=1}^{m+n} \beta_i k(x_i,\cdot)$ of Eq.~\eqref{eq:empirical_loss_objective}, by
\begin{align*}
    \norm{f_\z^{\lambda_s}-f_\z^{\lambda_t}}^2_{\widehat{\Hess}_{\lambda_t}(f_\z^{\lambda_t})} =
    \frac{1}{m+n}(\boldsymbol{\alpha}-\boldsymbol{\beta})^\text{T} {\bf K} {\bf E} {\bf K} (\boldsymbol{\alpha}-\boldsymbol{\beta})+ \lambda_t (\boldsymbol{\alpha}-\boldsymbol{\beta})^\text{T}{\bf K} (\boldsymbol{\alpha}-\boldsymbol{\beta}),
\end{align*}
\sloppy
where $\boldsymbol{\alpha}:=(\alpha_1,\ldots,\alpha_{m+n})$, $\boldsymbol{\beta}:=(\beta_1,\ldots,\beta_{m+n})$, ${\bf K}:=\left(k(x_i,x_j)\right)_{i,j=1}^{m+n}$ and ${\bf E}:=\mathrm{diag}(e_1,\ldots,e_{m+n})$ is a diagonal matrix with elements $e_{i}:=e^{-y_i\cdot \sum_{j=1}^{m+n} \beta_j k(x_i,x_j)}$, in case of the Exp method, and $e_{i}:=\frac{1-y_i}{2}$ in the case of KuLSIF, see Subsection~\ref{subsec:derivation_of_norm}.

To obtain theoretical error rate guarantees, Theorem~\ref{thm:balancing_principle_empirical} suggest to choose $M:=C$ as given by Eq.~\eqref{eq:const_in_empirical_risk_balancing_bound}.
However, as already noted in~\citep[Remark~6.1]{lu2020balancing}, this choice is too rough for the low number of samples considered in practical applications.
We therefore follow~\citep[Remark~6.1]{lu2020balancing} and choose $M=M_j:=\mathrm{Tr}(\widehat{\Hess}_{\lambda_j}(f_\z^{\lambda_j}))^{-2}$.

\begin{figure}[ht]
\includegraphics[width=\linewidth]{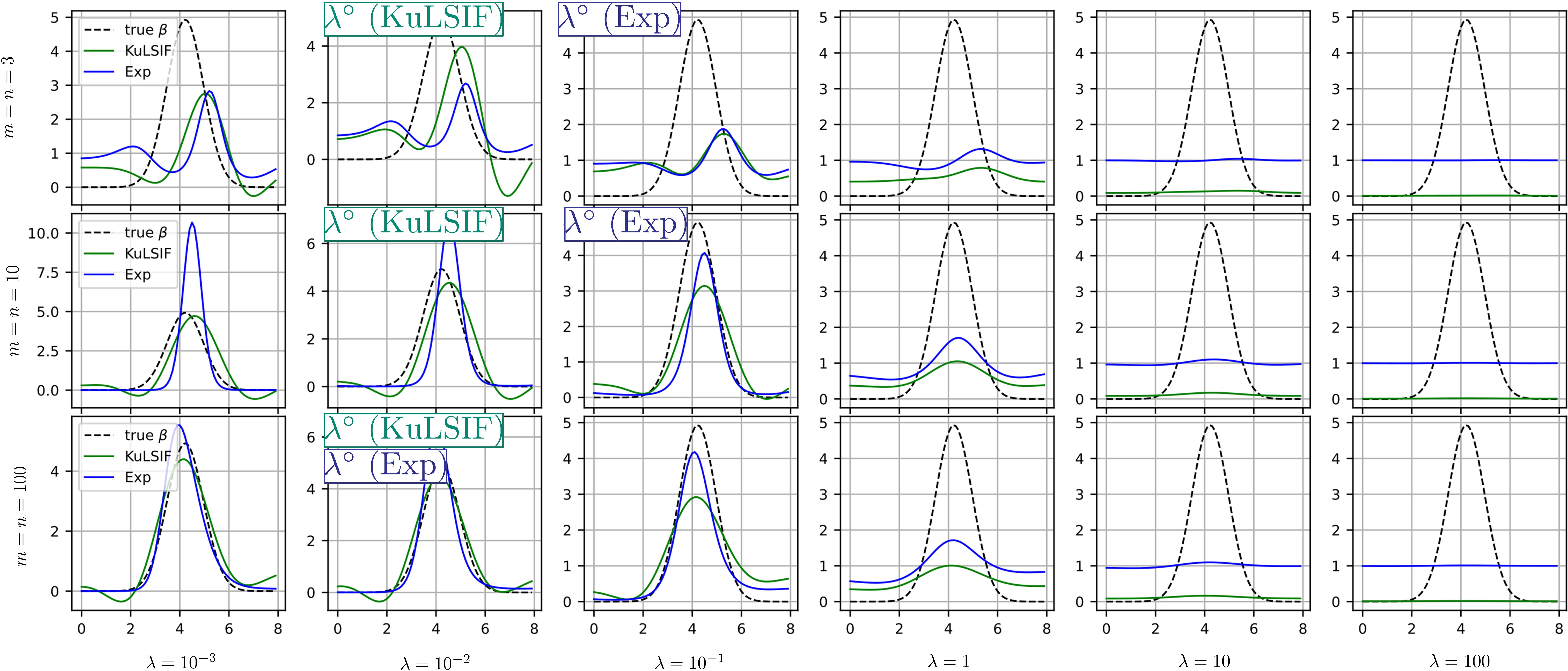}
\centering
\caption{Results for approximating the density ratio $\beta$ (black, dashed) by the KuLSIF (green, solid) and Exp method (blue, solid). Rows: sample sizes $m=n\in\{3,10,100\}$. Columns: Parameter choices $\lambda=10^i, i\in\{-3,\ldots,2\}$.
Choices of the proposed Lepskii type principle are marked by boxes in the upper left.}
\label{fig:lepskii}
\end{figure}

The results of our experiment can be found in Figure~\ref{fig:lepskii}.
The choice $\lambda^\circ$ in Eq.~\eqref{eq:numerical_example_bp} corresponds to $\lambda=10^{-2}$ for KuLSIF in all three cases of dataset size $m=n\in\{3,10,100\}$.
In case of the Exp method, it chooses $\lambda=10^{-1}$ for $m=n\in\{3,10\}$ and $\lambda=10^{-2}$ for $m=n=100$.
Although the choices are only optimal in two cases, w.r.t.~the mean squared error evaluated at a dense grid of values, the parameter choice principle in Eq.~\eqref{eq:numerical_example_bp} chooses always one of the two best values in the sequence.

\section{Verification of Details}
\label{sec:proofs}


\subsection{A priori parameter choice}
\label{subsec:quality_of_balancing_value}

The value $\lambda^\ast$ satisfying Eq.~\eqref{eq:balance} minimizes Eq.~\eqref{eq:bias_variance_simple} up to a constant factor (cf.~\citet{lu2013regularization}, Section~4.3).
To see this, let us denote a minimizer by $\lambda_1:=\argmin_{\lambda\geq 0}\eta S(m+n,\delta,\lambda)+A(\lambda)$. Then, either $\lambda^\ast<\lambda_1$, in which case
\begin{align*}
    \eta S(m+n,\delta,\lambda^\ast)=A(\lambda^\ast)\leq A(\lambda_1)\leq \eta \inf_{\lambda\geq 0}\left[ S(m+n,\delta,\lambda)+A(\lambda)\right],
\end{align*}
or $\lambda_1\leq \lambda^\ast$, and
\begin{align*}
    A(\lambda^\ast) = \eta S(m+n,\delta,\lambda^\ast)\leq \eta S(m+n,\delta,\lambda_1)\leq \eta \inf_{\lambda\geq 0} \left[ S(m+n,\delta,\lambda)+A(\lambda)\right],
\end{align*}
which gives, for large enough sample sizes and with high probability,
\begin{align}
\label{eq:optimality_of_exact_balancing_value}
    B(f_{\z}^{\lambda^\ast})-B(f_\mathcal{H})
    \leq \eta S(m+n,\delta,\lambda^\ast) + A(\lambda^\ast)
    \leq 2 \eta \inf_{\lambda\geq 0}  \left[ S(m+n,\delta,\lambda)+A(\lambda)\right].
\end{align}
We will need the following technical lemma.
\begin{lemma}
    \label{lemma:technical}
    If $A,B> 0$ and $k\geq 2 A \log(2 A B)$, then $A\log(B k)\leq k$.
\end{lemma}
\begin{proof}[Lemma~\ref{lemma:technical}]
    It holds that
    \begin{align*}
        A \log(B k)
        &=A \log\!\left(B k\frac{2 A B}{2 A B}\right)
        =A\left(\log\frac{B k}{2 A B} + \log(2 A B)\right)
        =A\left(\log\frac{k}{2 A} + \log(2 A B)\right)\\
        &\leq A\left(\frac{k}{2 A} + \log(2 A B)\right)
        = \frac{k}{2} + A\log(2 A B),
    \end{align*}
    which gives the desired result whenever $A\log(2 A B)\leq k/2$.
\end{proof}

The following proof of Proposition~\ref{prop:first_rate} follows the essential steps of~\citep[Corollary~4]{marteau2019beyond} applied to Eq.~\eqref{eq:regularized_Bregman_objective} with explicit treatment of the balancing property Eq.~\eqref{eq:balance} of $\lambda^\ast$.
\begin{proof}[Proposition~\ref{prop:first_rate}]
    \label{proof:balancing_value_rate}
    Using Assumptions~\ref{ass:iid_of_classification_data} and~\ref{ass:strictly_proper_composit_and_diff_bayes_risk}, we obtain from Lemma~\ref{lemma:bregman_form_of_loss}
    \begin{align*}
        B(f_{\z}^{\lambda})-B(f_\mathcal{H}) &=2 \left(\mathcal{R}(f_{\z}^{\lambda})-\mathcal{R}(f^\ast)\right) - 2 \left(\mathcal{R}(f_{\mathcal{H}})-\mathcal{R}(f^\ast)\right)\\
        &=2 \left(\mathcal{R}(f_{\z}^{\lambda})-\mathcal{R}(f_{\mathcal{H}})\right),%
    \end{align*}
    where $\mathcal{R}(f_{\z}^{\lambda})$ is an expectation w.r.t.~the loss function $\ell$ inducing $B(f_{\z}^{\lambda})=B_F(\beta,g(f_{\z}^{\lambda}))$ as defined in Lemma~\ref{lemma:bregman_form_of_loss}.
    Let $B_1,B_2$ be as defined in Lemma~\ref{lemma:marteau_theorem3}.
    Then, under Assumptions~\ref{ass:generalized_self_concordance} and~\ref{ass:boundedness_defin0_existence}, Lemma~\ref{lemma:marteau_theorem3} gives, for large enough sample sizes $m+n$ (as specified later), with probability larger than $1-2\delta$,
    \begin{align}
        B(f_{\z}^{\lambda})-B(f_\mathcal{H})&\leq S(m+n,\delta,\lambda)+A(\lambda)
        \label{eq:bias_variance_simple_proof}
    \end{align}
    for $S(m+n,\delta,\lambda)=168 \frac{B_1^2}{\lambda (m+n)}\log(\frac{2}{\delta})$ and $A(\lambda)=4\lambda \norm{f_\mathcal{H}}^2$.
    Solving
    \begin{align*}
        \eta \cdot S(m+n,\delta,\lambda^\ast) = A(\lambda^\ast)
    \end{align*}
    for $\lambda^\ast$ with $\eta=\frac{16^2 R^2}{42}\norm{f_\mathcal{H}}^2\geq 1$ yields
    \begin{align}
        \label{eq:exact_balance_value_simple_bound}
        \lambda^\ast = \frac{16 B_1 R}{\sqrt{m+n}} \log^\frac{1}{2}\!\left(\frac{2}{\delta}\right).
    \end{align}
    Let us now assume that $m+n$ is larger that
    {\small
    \begin{align}
    \label{eq:sample_condition_a_priori_bound}
        \max\!\left\{512\norm{f_{\mathcal{H}}}^2R^2\log\left(\frac{2}{\delta}\right), 512 \log\left(\frac{2}{\delta}\right), 
        \frac{256 B_1^2 R^2}{B_2^2}\log\!\left(\frac{2}{\delta}\right),
        \frac{9 B_2^2\log\!\left(\frac{6 B_2^2}{4 \delta B_1^2 R^2}\log^{-1}\!\left(\frac{2}{\delta}\right)\right)}{B_1^2 R^2\log\!\left(\frac{2}{\delta}\right)}
        \right\}.
    \end{align}}%
    In this case, $\lambda^\ast$ is an admissible value as specified by the conditions of Lemma~\ref{lemma:marteau_theorem3}:
    \begin{enumerate}
        \item The condition $\lambda^\ast\leq B_2$ is satisfied since $m+n\geq \frac{256 B_1^2 R^2}{B_2^2}\log\!\left(\frac{2}{\delta}\right)$.
        \item The third element of the set from which the maximum is taken in Eq.~\eqref{eq:sample_size_condition_lemma_marteau} requires
        \begin{align}
        \label{eq:proof_first_bound_split1}
            m+n\geq 24\frac{B_2}{\lambda^\ast}\log\!\left(\frac{8 B_2}{\lambda^\ast \delta}\right)
            = 3 c_1 (m+n)^{\frac{1}{2}}\log\!\left(\frac{1}{\delta} c_1 (m+n)^{\frac{1}{2}}\right)
        \end{align}
        with $c_1:=\frac{8 B_2}{\lambda^\ast (m+n)^{\frac{1}{2}}}=\frac{B_2}{2 B_1 R}\log^{-\frac{1}{2}}\!\left(\frac{2}{\delta}\right)$.
        Eq.~\eqref{eq:proof_first_bound_split1} is guaranteed by Lemma~\ref{lemma:technical} (with $k:=(m+n)^{\frac{1}{2}}$, $A:=3 c_1$ and $B:=\frac{1}{\delta} c$), whenever $2 A\log(2 A B)\leq k$. That is, whenever $m+n$ is larger than
        \begin{align*}
            36 c_1^2 \log^2\left(\frac{6 c_1^2}{\delta}\right)
            =\frac{9 B_2^2}{B_1^2 R^2}\log^{-1}\!\left(\frac{2}{\delta}\right)
            \log\!\left(\frac{6 B_2^2}{4 \delta B_1^2 R^2}\log^{-1}\!\left(\frac{2}{\delta}\right)\right).
        \end{align*}
        \item The last element from the set from which the maximum is taken in Eq.~\eqref{eq:sample_size_condition_lemma_marteau} requires $m+n\geq 256\frac{R^2 B_1^2}{ \left(\lambda^\ast\right)^2} \log\!\left(\frac{2}{\delta}\right)$, which is satisfied.
    \end{enumerate}
    Note that $\eta\cdot S(m+n,\delta,\lambda^\ast)+A(\lambda^\ast)$ upper bounds Eq.~\eqref{eq:bias_variance_simple_proof} evaluated at $\lambda^\ast$.
    The desired result follows with
        \begin{align}
        \label{eq:exact_form_of_constant_a_priori_bound}
        C:=128 \norm{f_\mathcal{H}}^2 B_1 R.
    \end{align}
\end{proof}

\subsection{Error Rate Requiring A Posteriori Parameter Choice}
\label{subsec:proof_a_posteriori_rate}

The following proof of Proposition~\ref{prop:fast_rate} uses Lemma~\ref{lemma:bregman_form_of_loss} and follows the proof of~\citep[Corollary~9]{marteau2019beyond} applied to Eq.~\eqref{eq:regularized_Bregman_objective} (noting that Lemma~\ref{lemma:bregman_form_of_loss} can be applied) with explicit treatment of the balancing property Eq.~\eqref{eq:balance} of $\lambda^\ast$.

\begin{proof}[Proposition~\ref{prop:fast_rate}]
\label{proof:proposition_fast_rate}
Let us denote by
$L:=\norm{\Hess(f_\mathcal{H})^{-r} f_\mathcal{H}}$.
Using Assumptions~\ref{ass:iid_of_classification_data} and~\ref{ass:strictly_proper_composit_and_diff_bayes_risk}, we obtain, from Lemma~\ref{lemma:bregman_form_of_loss},
    \begin{align*}
        B(f_{\z}^{\lambda})-B(f_\mathcal{H}) &=2 \left(\mathcal{R}(f_{\z}^{\lambda})-\mathcal{R}(f^\ast)\right) - 2 \left(\mathcal{R}(f_{\mathcal{H}})-\mathcal{R}(f^\ast)\right)\\
        &=2 \left(\mathcal{R}(f_{\z}^{\lambda})-\mathcal{R}(f_{\mathcal{H}})\right),%
    \end{align*}
    where $\mathcal{R}(f_{\z}^{\lambda})$ is an expectation w.r.t.~the loss function $\ell$ inducing $B(f_{\z}^{\lambda})=B_F(\beta,g(f_{\z}^{\lambda}))$ as defined in Lemma~\ref{lemma:bregman_form_of_loss}.
    Let $B_1^\ast,B_2^\ast$ be as defined in Lemma~\ref{lemma:marteau_main_result}.
    Then, under Assumptions~\ref{ass:generalized_self_concordance}--\ref{ass:capacity_condition}, Lemma~\ref{lemma:marteau_main_result} gives, for large enough sample sizes $m+n$ (as specified below in Eq.~\eqref{eq:requirement_sample_size_fast_rate_proposition}), with probability larger than $1-2\delta$,
    \begin{align}
        \label{eq:bias_variance_fast_proof}
        B(f_{\z}^{\lambda})-B(f_\mathcal{H})&\leq 2 S(m+n,\delta,\lambda)+2 A(\lambda)
    \end{align}
    for $S(m+n,\delta,\lambda)=414 \frac{Q_0^2}{\lambda^{1/\alpha} (m+n)}\log(\frac{2}{\delta})$ and $A(\lambda)=414 L^2 \lambda^{1+2 r}$.
Solving Eq.~\eqref{eq:balance} for $\lambda^\ast$ and $\eta= 1296 \log^{-1}\!\left(\frac{2}{\delta}\right)$ (notably being larger than $1$ since $\delta\geq 2 e^{-1296}$), gives
\begin{align}
    \label{eq:lambda_star_in_fast_rate_proof}
    \lambda^\ast=\left(\frac{1296 Q_0^2}{(m+n) L^2}\right)^{\frac{\alpha}{1+2 r\alpha +\alpha}}.
\end{align}
In what follows, we will show that $\lambda^\ast$ satisfies all assumptions from Lemma~\ref{lemma:marteau_main_result} whenever $m+n$ is larger than
\begin{align}
    \label{eq:requirement_sample_size_fast_rate_proposition}
    \begin{split}
        1296\cdot \max\!\Bigg\{&\frac{Q_0^2}{L^2 (B_2^\ast)^{1+2 r +\frac{1}{\alpha}}},
        \frac{Q_0^2}{L^2} \left(2 L R \log\frac{2}{\delta}\right)^{\frac{\alpha+2\alpha r+1}{\alpha r}},
        \frac{1296  (B_1^\ast)^{2+4\alpha r+2\alpha}}{L^2 Q_0^{4\alpha r+2\alpha} (B_2^\ast)^{1+2 \alpha r+\alpha}},\\
        &\frac{L^2 \left(16\cdot 648 B_2^\ast\right)^{1+2 r+\frac{1}{\alpha}}}{1296^2 Q_0^2}
        \log^{1+2 r+\frac{1}{\alpha}}\!\left(\frac{48 (B_2^\ast)^2 1296\cdot 414^2}{\delta} \left(\frac{1296 Q_0^2 }{L^2}\right)^{-\frac{2 \alpha}{1+2 r\alpha +\alpha}}\right)
        \Bigg\}.
    \end{split}
\end{align}
\begin{enumerate}
    \item The requirement $\lambda^\ast\leq B_2^\ast$ is equivalent to
    \begin{align*}
        m+n \geq \frac{1296 Q_0^2 }{L^2 (B_2^\ast)^{1+2 r +\frac{1}{\alpha}}}.
    \end{align*}
    \item The requirement $\lambda^\ast\leq (2 L R \log\frac{2}{\delta})^{-1/r}$ is equivalent to
    \begin{align*}
        m+n\geq \frac{1296 Q_0^2}{L^2} \left(2 L R \log\frac{2}{\delta}\right)^{\frac{\alpha+2\alpha r+1}{\alpha r}}.
    \end{align*}
    \item The requirement $\lambda^\ast\leq Q_0^{2\alpha} (B_2^\ast)^\alpha (B_1^\ast)^{-2 \alpha}$ is equivalent to
    \begin{align*}
        m+n\geq \frac{1296  (B_1^\ast)^{2+4\alpha r+2\alpha}}{L^2 Q_0^{4\alpha r+2\alpha} (B_2^\ast)^{1+2 \alpha r+\alpha}}.
    \end{align*}
    \item The requirement $m+n\geq \frac{1296 Q_0^2}{L^2 (\lambda^\ast)^{1+2 r+1/\alpha}}$ follows directly from the definition of $\lambda^\ast$.
    \item The last requirement is
    \begin{align}
        m+n &\geq 5184 \frac{B_2^\ast}{\lambda^\ast}\log\left(\frac{8 \cdot 414^2 B_2^\ast}{\lambda^\ast \delta}\right)\nonumber\\
        \label{eq:proof_fast_rate_technical_mn_condition1}
        &= 648 (m+n)^{\frac{\alpha}{1+2 r\alpha +\alpha}} c_1 \log\!\left(\frac{414^2}{\delta} c_1 (m+n)^{\frac{\alpha}{1+2 r\alpha +\alpha}}\right)
    \end{align}
    with $c_1:=8 B_2^\ast\left(\frac{1296 Q_0^2 }{L^2}\right)^{-\frac{\alpha}{1+2 r\alpha +\alpha}}$. Eq.~\eqref{eq:proof_fast_rate_technical_mn_condition1} follows from Lemma~\ref{lemma:technical} ($k:=(m+n)^{\frac{\alpha}{1+2 r\alpha +\alpha}}$, $A:=648 c_1$ and $B:=\frac{414^2}{\delta} c_1$), whenever
    \begin{align*}
        m+n\geq
        \frac{L^2 \left(16\cdot 648 B_2^\ast\right)^{1+2 r+\frac{1}{\alpha}}}{1296 Q_0^2 }
        \log^{1+2 r+\frac{1}{\alpha}}\!\left(\frac{64 (B_2^\ast)^2 1296\cdot 414^2}{\delta} \left(\frac{1296 Q_0^2 }{L^2}\right)^{-\frac{2 \alpha}{1+2 r\alpha +\alpha}}\right),
    \end{align*}
    since $(m+n)^{\frac{\alpha}{1+2 r\alpha +\alpha}}\leq (m+n)^{\frac{1}{2}}$.
    The latter follows from $\alpha\leq \frac{1}{1-2 r}$.
\end{enumerate}
Evaluating 
$2 \eta\cdot S(m+n,\delta,\lambda^\ast)+2 A(\lambda^\ast)$, which upper bounds 
Eq.~\eqref{eq:bias_variance_fast_proof} evaluated at $\lambda^\ast$, gives the desired result for
        \begin{align}
        \label{eq:exact_form_of_constant_fast_bound}
        C:= 2\cdot 828  L^2 \left(\frac{1296 Q_0^2}{L^2}\right)^{\frac{\alpha+2 r\alpha}{1+2\alpha r + \alpha}}.
    \end{align}

\end{proof}

\subsection{Parameter Choice with Known Norm}
\label{subsec:proofs_balancing_principle_known_norm}

We first prove the result assuming access to the norm $\norm{\cdot}_{\Hess_\lambda(f_\mathcal{H})}$.
\begin{proof}[Theorem~\ref{thm:balancing_principle_known_norm}]
    Using Assumptions~\ref{ass:iid_of_classification_data} and~\ref{ass:strictly_proper_composit_and_diff_bayes_risk}, we obtain, from Lemma~\ref{lemma:bregman_form_of_loss},
    \begin{align*}
        B(f_{\z}^{\lambda_+})-B(f_\mathcal{H}) &=2 \left(\mathcal{R}(f_{\z}^{\lambda_+})-\mathcal{R}(f^\ast)\right) - 2 \left(\mathcal{R}(f_{\mathcal{H}})-\mathcal{R}(f^\ast)\right)\\
        &=2 \left(\mathcal{R}(f_{\z}^{\lambda_+})-\mathcal{R}(f_{\mathcal{H}})\right),%
    \end{align*}
    where $\mathcal{R}(f_{\z}^{\lambda_+})$ is an expectation w.r.t.~the loss function $\ell$ inducing $B(f_{\z}^{\lambda_+})=B_F(\beta,f_{\z}^{\lambda_+})$ as defined in Lemma~\ref{lemma:bregman_form_of_loss}.
    From Lemma~\ref{lemma:marteau_main_result} we further know that, with probability at least $1-2\delta$,
    \begin{align}
    B(f_{\z}^{\lambda_+})-B(f_\mathcal{H})
        &\leq
        2 \norm{f_{\z}^{\lambda_+}-f_{\mathcal{H}}}_{\Hess(f_{\mathcal{H}})}^2,\nonumber
    \end{align}
    whenever
    \begin{align}
    \label{eq:sample_size_known_norm}
        m+n\geq \max\left\{5184 \frac{B_2^\ast}{\lambda_0}\log\left(\frac{8 \cdot 414^2 B_2^\ast}{\lambda_0 \delta}\right), \frac{1296 Q_0^2}{L^2 \lambda_0^{1+2 r+1/\alpha}} \right\}.
    \end{align}
    We now define
    \begin{align*}
        \overline{\lambda}:=\max\{\lambda_i:A(\lambda_i) &\leq \eta\cdot S(m+n,\delta,\lambda_i)\}
    \end{align*}
    with $S(m+n,\delta,\lambda):=\frac{414 Q_0^2}{(m+n)\lambda^{1/\alpha}}\log\!\left(\frac{2}{\delta}\right)$, $A(\lambda):=414 L^2\lambda^{1+2 r}$ as used in Lemma~\ref{lemma:marteau_main_result} and $\eta:=1296\log^{-1}\!\left(\frac{2}{\delta}\right)$ as used in Proposition~\ref{prop:fast_rate}.
    Using $(a+b)^2\leq 2 a^2+2 b^2$ and the triangle inequality, we obtain
    \begin{align}
    B(f_{\z}^{\lambda_+})-B(f_\mathcal{H})
        &\leq
        4 \norm{f_{\z}^{\lambda_+}-f_{\z}^{\overline{\lambda}}}_{\Hess(f_{\mathcal{H}})}^2 + 4\norm{f_{\z}^{\overline{\lambda}}-f_{\mathcal{H}}}_{\Hess(f_{\mathcal{H}})}^2\nonumber\\
        &\leq
        4 \norm{f_{\z}^{\lambda_+}-f_{\z}^{\overline{\lambda}}}_{\Hess_{\overline{\lambda}}(f_{\mathcal{H}})}^2 + 4\norm{f_{\z}^{\overline{\lambda}}-f_{\mathcal{H}}}_{\Hess_{\overline{\lambda}}(f_{\mathcal{H}})}^2,
        \label{eq:proof_known_norm_helper1}
    \end{align}
    where the last inequality follows from the fact that $\Hess_\lambda(f)=\Hess(f)+\lambda I\succcurlyeq \Hess(f)$ and it holds with probability at least $1-2\delta$, for large enough sample size.
    
    We will now prove bounds on the two terms in Eq.~\eqref{eq:proof_known_norm_helper1} separately.
    To bound the first term, we note that $\overline{\lambda}\leq\lambda^\ast$ since $A(\lambda)$ is increasing and $S(m+n,\delta,\lambda)$ is decreasing in $\lambda$.
    Eq.~\eqref{eq:necessary_balancing_condition} implies, for any $\lambda_j\leq \overline{\lambda}\leq\lambda^\ast$,
    \begin{align*}
        \norm{f_{\z}^{\overline{\lambda}}-f_{\z}^{\lambda_j}}_{\Hess_{\lambda_j}(f_{\mathcal{H}})}^2\leq 8\eta\cdot S(m+n,\delta,\lambda_j),
    \end{align*}
    which holds with probability at least $1-2l\delta$ for all $j$ at the same time.
    As a consequence, we get $\overline{\lambda}\leq \lambda_+$ from the maximizing property of $\lambda_+$ in Eq.~\eqref{eq:balancing_principle_estimate_known_norm}.
    Another consequence of $\overline{\lambda}\leq \lambda_+$ and Eq.~\eqref{eq:balancing_principle_estimate_known_norm} is that
    \begin{align}
    \label{eq:proof_helper2}
        \norm{f_{\z}^{\lambda_+}-f_{\z}^{\overline{\lambda}}}_{\Hess_{\overline{\lambda}}(f_{\mathcal{H}})}^2
        \leq 8\eta\cdot S(m+n,\delta,\overline{\lambda}).
    \end{align}
    The bound
    \begin{align}
    \label{eq:proof_helper3}
        \norm{f_{\z}^{\overline{\lambda}}-f_{\mathcal{H}}}_{\Hess_{\overline{\lambda}}(f_{\mathcal{H}})}^2
        \leq S(m+n,\delta,\overline{\lambda}) + A(\overline{\lambda})
        &\leq 2\eta\cdot S(m+n,\delta,\overline{\lambda})
    \end{align}
    follows from Lemma~\ref{lemma:marteau_main_result} and the definition of $\overline{\lambda}$.
    Eq.~\eqref{eq:proof_helper3} holds with probability at least $1-2\delta$.
    Inserting the bounds Eq.~\eqref{eq:proof_helper2} and Eq.~\eqref{eq:proof_helper3} into Eq.~\eqref{eq:proof_known_norm_helper1} gives
    \begin{align*}
        B(f_{\z}^{\lambda_+})-B(f_\mathcal{H})
        \leq 40\eta\cdot S(m+n,\delta,\overline{\lambda}),
    \end{align*}
    which holds with probability at least $1-(4+2l)\delta$.
    Without loss of generality assume that $\lambda_0 \cdot\xi^s=\overline{\lambda}\leq\lambda^\ast=\lambda_0 \cdot\xi^{s+1}$.
    It follows that $\lambda^\ast\leq \overline{\lambda}\cdot \xi$ and $\xi^{1/\alpha} S(m+n,\delta,\lambda^\ast)\geq S(m+n,\delta,\overline{\lambda})$.
    Consequently,
    \begin{align*}
        B(f_{\z}^{\lambda_+})-B(f_\mathcal{H})
        &\leq 40 \xi^{1/\alpha}\eta\cdot S(m+n,\delta,\lambda^\ast)
        \leq C (m+n)^{-\frac{2 r \alpha+\alpha}{2 r \alpha+\alpha+1}}
    \end{align*}
    holds with probability at least $1-(4+2l)\delta$ and with
    \begin{align}
        \label{eq:const_in_known_risk_balancing_bound}
        C:=16560 \xi^{1/\alpha}  L^2 \left(\frac{1296 Q_0^2}{L^2}\right)^{\frac{\alpha+2 r\alpha}{1+2\alpha r + \alpha}}.
    \end{align}
\end{proof}

\subsection{Parameter Choice with Empirical Norm}
\label{subsec:proofs_balancing_principle_empirical}

In the following, let $\mathcal{R}_\lambda(f):= \mathcal{R}(f)+ \frac{\lambda}{2}\norm{f}^2$, $f^\lambda:=\argmin_{f\in\mathcal{H}} \mathcal{R}_\lambda(f)$
and $\ttt(f):=\sup_{z\in\mathrm{supp}(\rho)}\sup_{g\in\varphi(z)}|f\cdot g|$ with $\ell_z, Z, \rho, \varphi$ as in Assumptions~\ref{ass:iid_of_classification_data},~\ref{ass:boundedness_defin0_existence} and~\ref{ass:generalized_self_concordance}.
Recall also the definitions $$\Hess_\lambda(f)=\Hess(f)+\lambda I=\mathbb{E}[\nabla^2 \ell_Z(f)]+\lambda I
\quad \text{and} \quad \widehat{\Hess}(f)=\frac{1}{m+n}\sum_{z\in\z} \nabla^2\ell_z(f).$$
We need the following three lemmas.
\begin{lemma}[{\citet{marteau2019beyond}, Proposition~15}]
    \label{lemma:marteau_prop15}
    Let Assumptions~\ref{ass:iid_of_classification_data},~\ref{ass:boundedness_defin0_existence} and~\ref{ass:generalized_self_concordance} be satisfied, $\lambda\geq 0$ and $f_1,f_2\in\mathcal{H}$. Then, we have
    \begin{align}
        \label{eq:marteau_prop15_point1}
        \Hess_\lambda(f_1) &\preccurlyeq e^{\ttt(f_1-f_2)} \Hess_\lambda(f_2)\\
        \label{eq:marteau_prop15_point2}
        \mathcal{R}_\lambda(f_1)-\mathcal{R}_\lambda(f_0) - \nabla \mathcal{R}_\lambda(f_0)(f_1-f_0) &\leq \psi(\ttt(f_1-f_0))\norm{f_1-f_0}_{\Hess_\lambda(f_0)}^2
    \end{align}
    with $\psi(t)=(e^t-t-1)/t^2$.
\end{lemma}

\begin{lemma}[{\citet{marteau2019beyond}}]
    \label{lemma:t_concentration}
    Under the conditions of Lemma~\ref{lemma:marteau_main_result}, we have, with probability at least $1-2\delta$,
    \begin{align}
        \label{eq:t_concentration}
        \ttt(f_\mathcal{H}-f^\lambda)\leq \log(2),
        \quad \ttt(f^\lambda-f_\z^{\lambda})\leq \log(2).
    \end{align}
\end{lemma}
\begin{proof}
    Lemma~33 in~\citep{marteau2019beyond} shows that (noting that $\max\{\mathrm{df}_\lambda,(Q_0^\ast)^2\}\leq (B_1)^2/\lambda$ from their Proposition~18 and $1/r_\lambda(\theta^\ast)\leq R/\sqrt{\lambda}$ by the remark after their Assumption~6) the conditions on $n$ in Lemma~\ref{lemma:marteau_main_result} allow to apply their Proposition~16, which gives the desired result.
\end{proof}
\begin{lemma}[{\citet{marteau2019beyond}, \citet{rudi2017generalization}}]
    \label{lemma:technical_concentration_proof}
    Let the conditions of Lemma~\ref{lemma:marteau_main_result} be satisfied. Then, it holds, with probability at least $1-\delta$,
    \begin{align}
        \label{eq:technical_lemma1}
        \Hess_\lambda(f)
        \preccurlyeq 2 \widehat{\Hess}_\lambda(f).
    \end{align}
    If, in addition, $0<\norm{\Hess(f)}$, then it holds, for all
    \begin{align}
        \label{eq:additional_sample_size_condition_rudi_etc}
        m+n\geq \frac{16 B_2^\ast}{\norm{\Hess(f)}}\log\!\left(\frac{2}{\delta}\right),
    \end{align}
    with probability at least $1-\delta$,
    \begin{align}
        \label{eq:technical_lemma2}
        \widehat{\Hess}_\lambda(f) \preccurlyeq \frac{3}{2} \Hess_\lambda(f).
    \end{align}
\end{lemma}
\begin{proof}
    Eq.~\eqref{eq:technical_lemma1} is proven in~\citep[Lemma~30]{marteau2019beyond}.
    Eq.~\eqref{eq:technical_lemma2} follows from~\citep[Proposition~47]{marteau2019beyond} (attributed to~\citep{rudi2017generalization}, Proposition~6 and Proposition~8) and~\citep[Remark~48]{marteau2019beyond}.
\end{proof}
We are now ready to provide more precise steps on the derivation of Lemma~\ref{lemma:marteau_main_result}, which is essentially a re-phrasing of the proof of~\citep[Theorem~38]{marteau2019beyond} in our notation.
\begin{proof}[Lemma~\ref{lemma:marteau_main_result}]
    From Eq.~\eqref{eq:marteau_prop15_point2} and the fact that $\Hess(f)\preccurlyeq\Hess_\lambda(f)$, we obtain
    \begin{align}
        \mathcal{R}(f_{\z}^{\lambda})-\mathcal{R}(f_{\mathcal{H}}) &\leq \psi(\ttt(f_{\z}^{\lambda}-f_\mathcal{H})) \norm{f_{\z}^{\lambda}-f_{\mathcal{H}}}_{\Hess(f_{\mathcal{H}})}^2\nonumber\\
        &\leq \psi(\ttt(f_{\z}^{\lambda}-f_\mathcal{H})) \norm{f_{\z}^{\lambda}-f_{\mathcal{H}}}_{\Hess_\lambda(f_{\mathcal{H}})}^2\label{eq:blabla},
    \end{align}
    which is exactly how the proof of~\citep[Theorem~38]{marteau2019beyond} works.
    Indeed, they proceed by showing how Eq.~\eqref{eq:blabla} can be upper bounded by $414 \frac{Q_0^2}{(m+n) \lambda^{1/\alpha}}\log\left(\frac{2}{\delta}\right) + 414 L^2 \lambda^{1+2 r}$.
    To obtain Lemma~\ref{lemma:marteau_main_result}, we just remark that $\ttt(f_{\z}^{\lambda}-f_\mathcal{H})\leq \ttt(f_{\z}^{\lambda}-f^\lambda)+\ttt(f_{\z}^{\lambda}-f^\lambda))\leq 2\log(2)$ follows from Lemma~\ref{lemma:t_concentration}.
    We further get $\psi(\ttt(f_{\z}^{\lambda}-f_\mathcal{H}))\leq \psi(2 \log(2))\leq 1$ since $\psi(t)$ increases with $t$.
\end{proof}

Lemma~\ref{lemma:Hessian_concentration} can be obtained as follows.
\begin{proof}[Lemma~\ref{lemma:Hessian_concentration}]
    From Lemma~\ref{lemma:technical_concentration_proof} we get, with probability at least $1-\delta$,
    \begin{align*}
    \widehat{\Hess}_\lambda(f_\z^\lambda)
    \preccurlyeq
    \frac{3}{2}\Hess_\lambda(f_\mathcal{H}).
    \end{align*}
    From Lemma~\ref{lemma:marteau_prop15}, we further obtain
    \begin{align*}
        \frac{3}{2}\Hess_\lambda(f_\mathcal{H})
        \preccurlyeq
        \frac{3}{2} e^{\ttt(f_\z^\lambda-f_\mathcal{H})} \Hess_\lambda(f_\mathcal{H})
        \preccurlyeq
        \frac{3}{2} e^{\ttt(f_\z^\lambda-f^\lambda)+\ttt(f^\lambda-f_\mathcal{H})} \Hess_\lambda(f_\mathcal{H})
        \preccurlyeq
        6 \Hess_\lambda(f_\mathcal{H}),
    \end{align*}
    where the last inequality uses also Lemma~\ref{lemma:t_concentration}.
    Proceeding with Eq.~\eqref{eq:technical_lemma1} in Lemma~\ref{lemma:technical_concentration_proof} gives
    \begin{align*}
        6 \Hess_\lambda(f_\mathcal{H})
        \preccurlyeq
        12 \widehat{\Hess}_\lambda(f_\mathcal{H})
        \preccurlyeq
        12 e^{\ttt(f_\z^\lambda-f_\mathcal{H})} \widehat{\Hess}_\lambda(f_\z^\lambda)
        \preccurlyeq
        48 \widehat{\Hess}_\lambda(f_\z^\lambda).
    \end{align*}
    The conclusion follows whenever the above inequalities are satisfied at the same time, which happens with probability at least $1-4\delta$.
\end{proof}

We are now ready to prove our main result.
\begin{proof}[Theorem~\ref{thm:balancing_principle_empirical}]
    Following the proof of Theorem~\ref{thm:balancing_principle_known_norm}
    line by line up to Eq.~\eqref{eq:proof_known_norm_helper1} and replacing $\lambda_+$ by $\lambda_\mathrm{BP}$ gives, with probability at least $1-2\delta$,
    \begin{align}
        \label{eq:first_stage_proof_empirical}
        B(f_\z^{\lambda_\mathrm{BP}})-B(f_\mathcal{H})
        \leq
        4 \norm{f_{\z}^{\lambda_\mathrm{BP}}-f_{\z}^{\overline{\lambda}}}_{\Hess_{\overline{\lambda}}(f_{\mathcal{H}})}^2 + 4\norm{f_{\z}^{\overline{\lambda}}-f_{\mathcal{H}}}_{\Hess_{\overline{\lambda}}(f_{\mathcal{H}})}^2,
    \end{align}
    whenever
    \begin{align}
    \label{eq:sample_size_empirical_norm}
        m+n\geq \max\left\{5184 \frac{B_2^\ast}{\lambda_0}\log\left(\frac{8 \cdot 414^2 B_2^\ast}{\lambda_0 \delta}\right), \frac{1296 Q_0^2}{L^2 \lambda_0^{2+1/\alpha}},\frac{16 B_2^\ast}{1296 b^\ast}\log\!\left(\frac{2}{\delta}\right) \right\},
    \end{align}
    where
    \begin{align*}
        \overline{\lambda}:=\max\{\lambda_i:A(\lambda_i) &\leq \eta\cdot S(m+n,\delta,\lambda_i)\},
    \end{align*}
    $S(m+n,\delta,\lambda):=\frac{414 Q_0^2}{(m+n)\lambda^{1/\alpha}}\log\!\left(\frac{2}{\delta}\right)$ and $A(\lambda):=414 L^2\lambda^{1+2 r}$ are defined as in Lemma~\ref{lemma:marteau_main_result}, and $\eta:=1296\log^{-1}\!\left(\frac{2}{\delta}\right)$ as defined in Proposition~\ref{prop:fast_rate}.
    Note that we used the requirement $\frac{1296 Q_0^2}{L^2\lambda_0^{2+1/\alpha}}\geq \frac{1296 Q_0^2}{L^2\lambda_0^{1+2 r+1/\alpha}}$ in Eq.~\eqref{eq:sample_size_empirical_norm} that is independent of $r$, and, we added a requirement based on $b^\ast$ from Assumption~\ref{ass:Hessian_of_sequence} to satisfy all conditions needed to apply Lemma~\ref{lemma:Hessian_concentration}.
    From Lemma~\ref{lemma:marteau_main_result} and the definition of $\overline{\lambda}$ we obtain a bound for the second term of Eq.~\eqref{eq:first_stage_proof_empirical}
    \begin{align}
    \label{eq:proof_helper_empirical3}
        \norm{f_{\z}^{\overline{\lambda}}-f_{\mathcal{H}}}_{\Hess_{\overline{\lambda}}(f_{\mathcal{H}})}^2
        \leq S(m+n,\delta,\overline{\lambda}) + A(\overline{\lambda})
        &\leq 2\eta\cdot S(m+n,\delta,\overline{\lambda}),
    \end{align}
    with holds again with probability at least $1-2\delta$.
    %
    
    To bound the first term in Eq.~\eqref{eq:first_stage_proof_empirical}, we note that Lemma~\ref{lemma:Hessian_concentration} implies that
    \begin{align}
        \label{eq:proof_heklper_empirical4}
        \norm{f_{\z}^{\lambda_\mathrm{BP}}-f_{\z}^{\overline{\lambda}}}_{\Hess_{\overline{\lambda}}(f_{\mathcal{H}})}^2
        \leq
        8 \norm{f_{\z}^{\lambda_\mathrm{BP}}-f_{\z}^{\overline{\lambda}}}_{\widehat{\Hess}_{\overline{\lambda}}(f_{\z}^{\overline{\lambda}})}^2,
    \end{align}
    which holds with probability at least $1-4l\delta$ for all possibilities $\lambda_\mathrm{BP},\overline{\lambda}\in\{\lambda_0,\ldots,\lambda_l\}$.
    Let us now focus on all $j$ which satisfy $\lambda_j\leq\overline{\lambda}$ and note that $\overline{\lambda}\leq\lambda^\ast$ by the monotonicity of $S(m+n,\delta,\lambda)$ and $A(\lambda)$.
    For the above mentioned $j$, we apply Lemma~\ref{lemma:Hessian_concentration}, which grants us, with probability at least $1-2 l\delta$ for each $j$ at the same time,
    \begin{align*}
        \norm{f_{\z}^{\overline{\lambda}}-f_{\z}^{\lambda_j}}_{\widehat{\Hess}_{\lambda_j}(f_{\z}^{\lambda_j})}^2
        \leq
        6 \norm{f_{\z}^{\overline{\lambda}}-f_{\z}^{\lambda_j}}_{\Hess_{\lambda_j}(f_\mathcal{H})}^2.
    \end{align*}
    Now, the same reasoning as for Eq.~\eqref{eq:necessary_balancing_condition} can be applied, to obtain, for any $\lambda_j\leq \overline{\lambda}\leq\lambda^\ast$,
    \begin{align*}
        \norm{f_{\z}^{\overline{\lambda}}-f_{\z}^{\lambda_j}}_{\widehat{\Hess}_{\lambda_j}(f_{\z}^{\lambda_j})}^2\leq 48\eta\cdot S(m+n,\delta,\lambda_j),
    \end{align*}
    which holds with probability at least $1-(2+2l)\delta$ resulting from the joint application of Lemma~\ref{lemma:marteau_main_result} and Lemma~\ref{lemma:Hessian_concentration}.
    As a consequence, we get $\overline{\lambda}\leq \lambda_\mathrm{BP}$ from the maximizing property of $\lambda_\mathrm{BP}$ in Eq.~\eqref{eq:balancing_principle_estimate}.
    Since $\overline{\lambda}\leq \lambda_\mathrm{BP}$ and Eq.~\eqref{eq:balancing_principle_estimate} holds, we get
    \begin{align}
    \label{eq:proof_helper_mepirical2}
        \norm{f_{\z}^{\lambda_\mathrm{BP}}-f_{\z}^{\overline{\lambda}}}_{\widehat{\Hess}_{\overline{\lambda}}(f_{\z}^{\overline{\lambda}})}^2
        \leq 48\eta\cdot S(m+n,\delta,\overline{\lambda}),
    \end{align}
    which realizes with probability at least $1-(2+2l)\delta$.
    Combining the results above gives
    \begin{align*}
        B(f_\z^{\lambda_\mathrm{BP}})-B(f_\mathcal{H})
        \leq 32\cdot 48\eta\cdot S(m+n,\delta,\overline{\lambda})+2\cdot 4\eta\cdot S(m+n,\delta,\overline{\lambda}) = 1544\eta\cdot S(m+n,\delta,\overline{\lambda}),
    \end{align*}
    which holds with probability at least $(6+6l)\delta$.
    With the same arguments as used in the proof of Theorem~\ref{thm:balancing_principle_known_norm}, it holds that
    $S(m+n,\delta,\overline{\lambda})\leq\xi^{1/\alpha} S(m+n,\delta,\lambda^\ast)$ and further
    \begin{align*}
        B(f_{\z}^{\lambda_\mathrm{BP}})-B(f_\mathcal{H})
        &\leq 1544\eta\cdot \xi^{1/\alpha} S(m+n,\delta,\lambda^\ast)
        \leq C (m+n)^{-\frac{2 r \alpha+\alpha}{2 r \alpha+\alpha+1}}
    \end{align*}
    with the constant
    \begin{align}
        \label{eq:const_in_empirical_risk_balancing_bound}
        C:=1296\cdot 1544 \xi^{1/\alpha}\cdot 828  L^2 \left(\frac{1296 Q_0^2}{L^2}\right)^{\frac{\alpha+2 r\alpha}{1+2\alpha r + \alpha}}.
    \end{align}
\end{proof}

\subsection{Derivation of Empirical Norms}
\label{subsec:derivation_of_norm}

\paragraph*{Exp method}
Let $\ell_{z}(f)=e^{-y f(x)}$ for $z=(x,y)$.
Then, the reproducing property of the RKHS $\mathcal{H}$ with kernel $k$ gives
\begin{align*}
    \nabla^2 \ell_z(f)= \nabla^2 e^{-y\cdot f(x)}
    = \nabla^2 e^{-y\cdot \langle k(x,\cdot), f\rangle} = e^{-y\cdot \langle k(x,\cdot), f\rangle} k(x,\cdot) k(x,\cdot).
\end{align*}
For the minimizer $f_\z^\lambda=\sum_{j=1}^{m+n} \alpha_j k(x_j,\cdot)$ of Eq.~\eqref{eq:empirical_loss_objective} as determined by the representer theorem, see, e.g.~\citep{scholkopf2001generalized}, it follows that
\begin{align*}
    \widehat{\Hess}_\lambda(f_\z^\lambda)=
    \frac{1}{m+n}\sum_{i=1}^{m+n}
    e^{-y_i\cdot \sum_{j=1}^{m+n} \alpha_j k(x_i,x_j)} k(x_i,\cdot) k(x_i,\cdot) +\lambda I.
\end{align*}
For $f_\z^{\lambda_s}=\sum_{j=1}^{m+n} \alpha_j k(x_j,\cdot)$ and $f_\z^{\lambda_t}=\sum_{j=1}^{m+n} \beta_j k(x_j,\cdot)$ we obtain
\begin{align*}
    \norm{f_\z^{\lambda_s}-f_\z^{\lambda_t}}^2_{\widehat{\Hess}_{\lambda_t}(f_\z^{\lambda_t})} &=
    \left\langle \widehat{\Hess}_{\lambda_t}^{\frac{1}{2}}(f_\z^{\lambda_t}) (f_\z^{\lambda_s}-f_\z^{\lambda_t}), \widehat{\Hess}_{\lambda_t}^{\frac{1}{2}}(f_\z^{\lambda_t}) (f_\z^{\lambda_s}-f_\z^{\lambda_t})\right\rangle\\
    &=
    \left\langle \widehat{\Hess}_{\lambda_t}(f_\z^{\lambda_t}) (f_\z^{\lambda_s}-f_\z^{\lambda_t}), f_\z^{\lambda_s}-f_\z^{\lambda_t}\right\rangle\\
    &=\left\langle \widehat{\Hess}_{\lambda_t}(f_\z^{\lambda_t}) \sum_{j=1}^{m+n}(\alpha_j-\beta_j) k(x_j,\cdot), \sum_{j=1}^{m+n}(\alpha_j-\beta_j) k(x_j,\cdot)\right\rangle\\
    &=\frac{1}{m+n}\sum_{i=1}^{m+n}\left\langle \sum_{t=1}^{m+n} (\alpha_t -\beta_t) k(x_i,x_t) e_{i} k(x_i,\cdot), \sum_{j=1}^{m+n}(\alpha_j-\beta_j) k(x_j,\cdot)\right\rangle\\
    &\phantom{=}~+~\lambda_t (\boldsymbol{\alpha}-\boldsymbol{\beta})^\text{T}{\bf K} (\boldsymbol{\alpha}-\boldsymbol{\beta})\\
     &=\frac{1}{m+n}
    (\boldsymbol{\alpha}-\boldsymbol{\beta})^\text{T} {\bf K} {\bf E} {\bf K} (\boldsymbol{\alpha}-\boldsymbol{\beta})+ \lambda_t (\boldsymbol{\alpha}-\boldsymbol{\beta})^\text{T}{\bf K} (\boldsymbol{\alpha}-\boldsymbol{\beta}),
\end{align*}
where $\boldsymbol{\alpha}:=(\alpha_1,\ldots,\alpha_{m+n})$, $\boldsymbol{\beta}:=(\beta_1,\ldots,\beta_{m+n})$, diagonal matrix ${\bf E}:=\mathrm{diag}(e_1,\ldots,e_{m+n})$, $e_{i}:=e^{-y_i\cdot \sum_{j=1}^{m+n} \beta_j k(x_i,x_j)}$ and ${\bf K}:=\left(k(x_i,x_j)\right)_{i,j=1}^{m+n}$.

\paragraph*{KuLSIF}

Let $\ell_{(x,-1)}(f)=\frac{1}{2} f(x)^2$ and $\ell_{(x,1)}(f)=- f(x)$.
Then, the reproducing property of the RKHS $\mathcal{H}$ with kernel $k$ gives
\begin{align*}
    \nabla^2 \ell_{(x,-1)}(f)=  k(x,\cdot) k(x,\cdot), \nabla^2 \ell_{(x,1)}(f)=  0.
\end{align*}
For the minimizer $f_\z^\lambda=\sum_{j=1}^{m+n} \alpha_j k(x_j,\cdot)$ of Eq.~\eqref{eq:empirical_loss_objective}, it follows that
\begin{align*}
    \widehat{\Hess}_\lambda(f_\z^\lambda)=
    \frac{1}{m+n}\sum_{i=1}^{m+n}
    \frac{1-y_i}{2} k(x_i,\cdot) k(x_i,\cdot) +\lambda I.
\end{align*}
For $f_\z^{\lambda_s}=\sum_{j=1}^{m+n} \alpha_j k(x_j,\cdot)$ and $f_\z^{\lambda_t}=\sum_{j=1}^{m+n} \beta_j k(x_j,\cdot)$ we get, analogously to the derivation for the Exp method,
\begin{align*}
    \norm{f_\z^{\lambda_s}-f_\z^{\lambda_t}}^2_{\widehat{\Hess}_{\lambda_t}(f_\z^{\lambda_t})}
     &=\frac{1}{m+n}
    (\boldsymbol{\alpha}-\boldsymbol{\beta})^\text{T} {\bf K} {\bf E} {\bf K} (\boldsymbol{\alpha}-\boldsymbol{\beta})+ \lambda_t (\boldsymbol{\alpha}-\boldsymbol{\beta})^\text{T}{\bf K} (\boldsymbol{\alpha}-\boldsymbol{\beta}),
\end{align*}
where $\boldsymbol{\alpha}:=(\alpha_1,\ldots,\alpha_{m+n})$, $\boldsymbol{\beta}:=(\beta_1,\ldots,\beta_{m+n})$, diagonal matrix ${\bf E}:=\mathrm{diag}(e_1,\ldots,e_{m+n})$, $e_{i}:=\frac{1-y_i}{2}$ and ${\bf K}:=\left(k(x_i,x_j)\right)_{i,j=1}^{m+n}$.

\section{Conclusion and Future Work}
We proposed error bounds for a large class of kernel methods for density ratio estimation for which no error bounds were known before (to the best of our knowledge).
The class consists of methods which minimize a regularized Bregman divergence over a reproducing kernel Hilbert space.
We also proposed a Lepskii type principle for choosing the regularization parameter.
The principle minimizes the error bounds without knowledge of the unknown regularity of the density ratio.
As a consequence it achieves a minimax optimal error rate in the special case of square loss.
The key proof technique for both, the error rates and the adaptive parameter choice, is a local quadratic approximation of the Bregman divergence, which is equivalent to an expected risk for a proper composite self-concordant loss function.
The minimax optimality of our error bounds for non-square losses is an open problem.

\section*{Acknowledgements}

We thank the anonymous reviewers for helpful comments and careful proofreading.
The research reported in this paper has been funded by the Federal Ministry for Climate Action, Environment, Energy, Mobility, Innovation and Technology (BMK), the Federal Ministry for Digital and Economic Affairs (BMDW), and the Province of Upper Austria in the COMET Module S3AI managed by the Austrian Research Promotion Agency FFG.

\bibliography{derivative}

\end{document}